\documentclass{article}

\usepackage{algorithm}
\usepackage{algorithmic}

\usepackage[margin=1.5in]{geometry}

\usepackage[utf8]{inputenc} 
\usepackage[T1]{fontenc}    
\usepackage{hyperref}       
\usepackage{url}            
\usepackage{booktabs}       
\usepackage{amsmath}
\usepackage{amsfonts}
\usepackage{dsfont}
\usepackage{amssymb}
\usepackage{amsthm}
\usepackage{nicefrac}       
\usepackage{microtype}      
\usepackage{xcolor}         
\usepackage{authblk}
\usepackage{graphicx}
\newtheorem{theorem}{Theorem}
\newtheorem{corollary}{Corollary}
\newtheorem{lemma}{Lemma}
\newtheorem*{theorem*}{Theorem}
\newtheorem*{corollary*}{Corollary}
\newtheorem*{lemma*}{Lemma}

\newtheorem{proposition}{Proposition}
\newtheorem{definition}{Definition}
\newtheorem{property}{Property}
\newtheorem*{property*}{Property}
\newtheorem{remark}{Remark}

\title{An Efficient Solution to s-Rectangular Robust Markov Decision Processes}

\author[1]{Navdeep Kumar}
\author[1]{Kfir Levy}
\author[2]{Kaixin Wang}
\author[1]{Shie Mannor}

\affil[1]{Technion}
\affil[2]{National University of Singapore}

\begin{document}
\maketitle
\begin{abstract}
We present an efficient robust value iteration for \texttt{s}-rectangular robust Markov Decision Processes (MDPs) with a time complexity comparable to standard (non-robust) MDPs which is significantly faster than any existing method. We do so by deriving the optimal robust Bellman operator in concrete forms using our $L_p$ water filling lemma. We unveil the exact form of the optimal policies, which turn out to be novel threshold policies with the probability of playing an action proportional to its advantage.
\end{abstract}

\section{Introduction}

In Markov Decision Processes (MDPs), an agent interacts with the environment and learns to optimally behave in it ~\cite{Sutton1998}.
However, the MDP solution may be very sensitive to little changes in the model parameters ~\cite{BiasVarianceShie}. Hence we should be cautious applying the solution of the MDP, when the model is changing or when there is uncertainty in the model parameters. Robust MDPs  provide a way to address this issue, where an agent can learn to optimally behave even when the model parameters are uncertain ~\cite{HKuhn2013,tamar14,Iyenger2005}.
Another motivation to study robust MDPs is that they can lead to better generalization \cite{robustnessAndGeneralization,genralization1,generalization2} compared to non-robust solutions.



Unfortunately, solving robust MDPs is proven to be NP-hard for general uncertainty sets \cite{RVI}. As a result, the uncertainty set is often assumed to be rectangular, which enables the existence of a contractive robust Bellman operators to obtain the optimal robust value function \cite{Nilim2005RobustCO,Iyenger2005,k-rectangularRMDP,r-rectRMDP,RVI}. Recently, there has been progress in solving robust MDPs for some \texttt{sa}-rectangular uncertainty sets via both value-based and policy-based methods \cite{Rcontamination,PG_RContamination}. An uncertainty set is said to be \texttt{sa}-rectangular if it can be expressed as a Cartesian product of the uncertainty in all states and actions. It can be further generalized to a \texttt{s}-rectangular uncertainty set if it can be expressed as a Cartesian product of the uncertainty in all states only.
Compared to \texttt{sa}-rectangular robust MDPs, \texttt{s}-rectangular robust MDPs are less conservative and hence more desirable; however, they are also much more difficult and poorly understood \cite{RVI}. Currently, there are few works that consider \texttt{s}-rectangular $L_p$ robust MDPs where uncertainty set is further constrained by $L_p$ norm, but they rely on black box methods which limits its applicability and offers little insights \cite{LinfPPI,ppi,derman2021twice,RVI}. No effective value or policy based methods exist for solving any \texttt{s}-robust MDPs. Moreover, it is known that optimal policies in \texttt{s}-rectangular robust MDPs can be stochastic, in contrast to \texttt{sa}-rectangular robust MDPs and non-robust MDPs \cite{RVI}. However, so far, nothing is known about the stochastic nature of the optimal policies in \texttt{s}-rectangular MDPs.
In this work, we mainly focus on \texttt{s}-rectangular $L_p$ robust MDPs. We first revise the unrealistic assumptions made in the noise transition kernel in \cite{derman2021twice} and introduce forbidden transitions, which leads to novel regularizers. Then we derive robust Bellman operator (policy evaluation) for a \texttt{s}-rectangular robust MDPs in closed form which is equivalent to reward-value-policy-regularized non-robust Bellman operator without radius assumption 5.1 in \cite{derman2021twice}. We exploit this equivalence to derive an optimal robust Bellman operator in concrete forms using our $L_p$-water pouring lemma which generalizes existing water pouring lemma for $L_2$ case \cite{anava2016k}. 
We can compute these operators in closed form for $p=1,\infty$ and exactly by a simple algorithm for $p=2$, and approximately by binary search for general $p$. We show that the time complexity of robust value iteration for $p=1,2$ is the same as that of non-robust value iteration. For general $p$, the complexity includes some additional log-factors due to binary searches.

In addition, we derive a complete characterization of  the stochastic nature of optimal policies in \texttt{s}-rectangular robust MDPs. The optimal policies in this case, are  threshold policies, that plays only actions with positive advantage with probability proportional to $(p-1)$-th power to its advantage.


\textbf{Related Work.} For \texttt{sa}-rectangular R-contamination robust MDPs, \cite{Rcontamination} derived robust Bellman operators which are equivalent to value-regularized-non-robust Bellman operators, enabling efficient robust value iteration. Building upon this work, \cite{PG_RContamination} derived robust policy gradient which is equivalvent to non-robust policy gradient with regularizer and correction terms. 
Unfortunately, these methods can't be naturally generalized to \texttt{s}-rectangular robust MDPs.

For \texttt{s}-rectangular robust MDPs,  methods such as robust value iteration \cite{Bagnell01solvinguncertain,RVI}, robust modified policy iteration \cite{Kaufman2013RobustMP}, partial robust policy iteration \cite{ppi} etc tries to approximately evaluate robust Bellman operators using variety of tools to estimate optimal robust value function.
The scalability of these methods has been limited due to their reliance on an external black-box solver such as Linear Programming.

Previous works have explored robust MDPs from a regularization perspective \cite{derman2021twice,DRO_derman,rewardRobustRegularization,MaxEntRL}. Specifically, \cite{derman2021twice} showed that \texttt{s}-rectangular robust MDPs is equivalent to reward-value-policy regularized MDPs, and  proposed a gradient based policy iteration for \texttt{s}-rectangular $L_p$ robust MDPs ( where uncertainty set is \texttt{s}-rectangular and constrained by $L_p$ norm). But this gradient based policy improvement relies on black box simplex projection, hence very slow and not scalable.

The detailed discussion of the above works can be found in the appendix.
\section{Preliminary}
\subsection{Notations}
For a set $\mathcal{S}$, $\lvert\mathcal{S}\rvert$ denotes its cardinality. $\langle u, v\rangle := \sum_{s\in\mathcal{S}}u(s)v(s)$ denotes the dot product between functions $u,v:\mathcal{S}\to\mathbb{R}$. $\lVert v\rVert_p^q :=(\sum_{s}\lvert v(s)\rvert^p)^{\frac{q}{p}}$ denotes the $q$-th power of $L_p$ norm of function $v$, and we use  $\lVert v\rVert_p := \lVert v\rVert^1_p$ and $\lVert v\rVert := \lVert v\rVert_2$ as shorthand. For a set $\mathcal{C}$, $\Delta_{\mathcal{C}}:=\{a:\mathcal{C} \to \mathbb{R}|a(c)\geq 0, \forall c, \sum_{c\in\mathcal{C}}a_c=1\}$ is the probability simplex over $\mathcal{C}$.
$\mathbf{0},\mathbf{1} $ denotes all zero vector and all ones vector/function respectively of appropriate dimension/domain. $\mathbf{1}(a=b):=1$ if $a=b$, 0 otherwise, is the indicator function. For vectors $u,v$, $\mathbf{1}(u\geq v)$ is component wise indicator vector, i.e. $\mathbf{1}(u\geq v)(x) = \mathbf{1}(u(x)\geq v(x))$. $A\times B =\{(a,b)\mid a\in A, b\in B\}$ is cartesain product between set $A$  and $B$.

\subsection{Markov Decision Processes}\label{sec:MDPs}
A Markov Decision Process (MDP) can be described as a tuple $(\mathcal{S},\mathcal{A},P,R,\gamma,\mu)$, where $\mathcal{S}$ is the state space, $\mathcal{A}$ is the action space, $P$ is a transition kernel mapping $\mathcal{S}\times\mathcal{A}$ to $\Delta_{\mathcal{S}}$, $R$ is a reward function mapping $\mathcal{S}\times\mathcal{A}$ to $\mathbb{R}$, $\mu$ is an initial distribution over states in $\mathcal{S}$, and $\gamma$ is a discount factor in $[0,1)$.
The expected discounted cumulative reward (return) is defined as
\begin{align*}
  \rho^\pi_{(P,R)} :=   &\mathbb{E}\left[\sum_{n=0}^{\infty}\gamma^n R(s_n,a_n)\Bigm| s_0\sim \mu, \pi,P\right].
\end{align*}
The return can be written compactly  as 
\begin{align}
   \rho^\pi_{(P,R)} =\langle \mu,v^\pi_{(P,R)}\rangle,
\end{align}
\cite{Puterman1994MarkovDP} where 
$v^\pi_{(P,R)}$ is the value function , defined as 
\begin{align}
  v^\pi_{(P,R)}(s) :=\mathbb{E}\left[\sum_{n=0}^\infty \gamma^nR(s_n,a_n)\Bigm\vert s_0= s,\pi,P\right].
\end{align}

Our objective is to find an optimal policy $\pi^*_{(P,R)}$ that maximizes the performance $\rho^\pi_{(P,R)}$. This performance can be written as : 
\begin{align}
    \rho^*_{(P,R)} := \max_{\pi}\rho^\pi_{(P,R)} = \langle \mu, v^*_{(P,R)}\rangle,
\end{align}
where $v^*_{(P,R)}:=\max_{\pi}v^\pi_{(P,R)}$ is the optimal value function \cite{Puterman1994MarkovDP}.


The value function $v^\pi_{(P,R)}$ and the optimal value function $v^*_{(P,R)}$ are the fixed points of the Bellman operator $\mathcal{T}^{\pi}_{(P,R)}$ and the robust Bellman operator $\mathcal{T}^*_{(P,R)}$, respectively \cite{Sutton1998}. These $\gamma$-contraction operators are defined as follows: For any vector $v$, and state $s\in \mathcal{S}$,
\begin{align*}
    (\mathcal{T}^\pi_{(P,R)} v)(s) &:= \sum_{a}\pi(a|s)\Bigm[R(s,a) +  \gamma \sum_{s'}P(s'|s,a)v(s')\Bigm], \quad \text{and}\\
    \mathcal{T}^*_{(P,R)} v &:= \max_{\pi}\mathcal{T}^\pi_{(P,R)} v.
\end{align*}

Therefore, the value iteration $v_{n+1} := \mathcal{T}^*_{(P,R)} v_n$ converges linearly to the optimal value function $v^*_{(P,R)}$. Given this optimal value function, the optimal policy can be computed as: $\pi^*_{(P,R)}  \in \text{arg}\max_{\pi}\mathcal{T}^{\pi}_{(P,R)} v^*_{(P,R)} $.

\begin{remark} The vector minimum of a set $U$ of vectors is defined component wise, i.e. $(\min_{u\in U}u)(i) := \min_{u\in U}u(i)$. This operation is well-defined only when there exists a minimal vector $u^*\in U$ such that $u^* \preceq u, \forall u\in U$. The same holds for other operations such as maximum, argmin, argmax, etc.
\end{remark}

\subsection{Robust Markov Decision Processes}
A robust Markov Decision Process (MDP) is a tuple $(\mathcal{S},\mathcal{A},\mathcal{P},\mathcal{R},\gamma,\mu)$ which generalizes the standard MDP by containing a set of transition kernels $\mathcal{P}$ and set of reward functions $\mathcal{R}$. Let uncertainty set $\mathcal{U} =\mathcal{P}\times\mathcal{R}$ be set of tuples of transition kernels and reward functions \cite{Iyenger2005,Nilim2005RobustCO}. 
The robust performance $\rho^\pi_\mathcal{U}$ of a policy $\pi$ is defined to be its worst performance on the entire uncertainty set $\mathcal{U}$ as 
\begin{align}\label{obj:RobPolEval}
    \rho^\pi_{\mathcal{U}} := \min_{(P,R)\in\mathcal{U}} \rho^\pi_{(P,R)}.
\end{align}
Our objective is to find an optimal robust policy $\pi^*_\mathcal{U}$ that maximizes the robust performance $\rho^\pi_{\mathcal{U}}$, defined as 
\begin{align}\label{obj:RobPolImp}
   \rho^*_{\mathcal{U}} := \max_{\pi}\rho^\pi_{\mathcal{U}} .
\end{align}
Solving the above robust objectives \ref{obj:RobPolEval} and \ref{obj:RobPolImp} are strongly NP-hard for general uncertainty sets, even if they are convex \cite{RVI}. Hence, the uncertainty set $\mathcal{U}=\mathcal{P}\times\mathcal{R}$ is commonly assumed to be \texttt{s}-rectangular, meaning that $\mathcal{R}$ and $\mathcal{P}$ can be decomposed state-wise as $\mathcal{R} = \times_{s \in \mathcal{S}} \mathcal{R}_{s}$ and $\mathcal{P} = \times_{s \in \mathcal{S}} \mathcal{P}_{s}$. For further simplification, $\mathcal{U}=\mathcal{P}\times\mathcal{R}$ is assumed to decompose state-action-wise as $\mathcal{R} = \times_{(s,a) \in \mathcal{S}\times\mathcal{A}} \mathcal{R}_{s,a}$ and $\mathcal{P} = \times_{(s,a) \in \mathcal{S}\times \mathcal{A}} \mathcal{P}_{s,a}$, known as $\texttt{sa}$-rectangular uncertainty set. Throughout the paper, the uncertainty set is assumed to be \texttt{s}-rectangular (or \texttt{sa}-rectangular) unless stated otherwise. Under the $\texttt{s}$-rectangularity assumption, for every policy $\pi$, there exists a robust value function  $v^\pi_{\mathcal{U}}$ which is the minimum of $v^\pi_{(P,R)}$ for all $(P,R) \in \mathcal{U}$, and the optimal robust value function $v^*_{\mathcal{U}}$ which is the maximum of $v^\pi_{\mathcal{U}}$ for all policies $\pi$ \cite{RVI}, that is
\[v^\pi_{\mathcal{U}} := \min_{{(P,R)\in\mathcal{U}}}v^\pi_{(P,R)}, \quad \text{and}\quad v^*_{\mathcal{U}} := \max_{\pi}v^\pi_\mathcal{U}.\]
This implies, robust policy performance can be rewritten as 
\begin{align*}
   \rho^\pi_{\mathcal{U}}  =  \langle \mu,v^\pi_{\mathcal{U}}\rangle , \quad \text{and}\quad \rho^*_{\mathcal{U}} = \langle \mu,v^*_{\mathcal{U}}\rangle.
\end{align*}
Furthermore, the robust value function $v^\pi_{\mathcal{U}}$ is the fixed point of the robust Bellmen operator $\mathcal{T}^\pi_{\mathcal{U}}$  \cite{RVI,Iyenger2005}, defined as 
\begin{align*}
    &(\mathcal{T}^\pi_{\mathcal{U}} v)(s) :=\min_{(P,R)\in\mathcal{U}}\sum_{a}\pi(a|s)\left[R(s,a) + \gamma \sum_{s'}P(s'|s,a)v(s')\right],
\end{align*}
and the optimal robust value function $v^*_{\mathcal{U}}$  is the fixed point of the optimal robust Bellman operator $\mathcal{T}^*_{\mathcal{U}}$ \cite{Iyenger2005,RVI}, defined as
\begin{align*}
    \mathcal{T}^*_{\mathcal{U}}v := \max_{\pi}\mathcal{T}^\pi_{\mathcal{U}}v.
\end{align*}
The optimal robust Bellman operator $\mathcal{T}^*_{\mathcal{U}}$ and  robust Bellman operators $\mathcal{T}^\pi_{\mathcal{U}}$ are $\gamma$ contraction maps for all policy $\pi$ \cite{RVI}, that is 
\begin{align*}
    &\lVert\mathcal{T}^*_{\mathcal{U}}v - \mathcal{T}^*_{\mathcal{U}}u\rVert_{\infty} \leq \gamma \lVert u-v\rVert_{\infty},\qquad \lVert\mathcal{T}^\pi_{\mathcal{U}}v - \mathcal{T}^\pi_{\mathcal{U}}u\rVert_{\infty} \leq \gamma \lVert u-v\rVert_{\infty},\qquad \forall \pi,u,v.
\end{align*}
So for all initial values $v^\pi_0,v^*_0$, sequences defined as 
\begin{align}
    v^\pi_{n+1} := \mathcal{T}^\pi_{\mathcal{U}}v^\pi_n , \qquad v^*_{n+1} := \mathcal{T}^*_{\mathcal{U}}v^*_n
\end{align}
converges linearly to their respective fixed points, that is $v^\pi_n\to v^\pi_{\mathcal{U}}$ and $v^*_n\to v^*_{\mathcal{U}}$. Given this optimal robust value function, the optimal robust policy can be computed as: $\pi^*_{\mathcal{U}} \in \text{arg}\max_{\pi}\mathcal{T}^{\pi }_{\mathcal{U}} v^*_{\mathcal{U}} $ \cite{RVI}. This makes the robust value iteration an attractive method for solving \texttt{s}-rectangular robust MDPs.

\section{Method}

In this section, we consider constraining the uncertainty set around nominal values by the $L_p$ norm, which is a natural way of limiting the broad class of \texttt{s} (or \texttt{sa})-rectangular uncertainty sets \cite{derman2021twice,ppi,UCRL2}.
We will then derive robust Bellman operators for these uncertainty sets, which can be used to obtain robust value functions. This will be done separately for \texttt{sa}-rectangular in Subsection \ref{main:sec:saLp} and \texttt{s}-rectangular case in Section \ref{main:sec:sLp}.

We begin by making a few useful definitions. We reserve  $q$ for Holder conjugate of $p$, i.e. $\frac{1}{p} + \frac{1}{q} = 1$. 
Let $p$-variance function  $\kappa_p:\mathcal{S}\to\mathbb{R}$  be defined as
\begin{equation}\label{def:kp}
    \kappa_p(v) :=\min_{\omega\in\mathbb{R}}\lVert v-\omega\mathbf{1}\lVert_p.
\end{equation}
For $p=1,2,\infty$, the $p$-variance function $\kappa_p$ has intuitive closed forms as summarized in Table \ref{tb:kappa}. For general $p$, it can be calculated by binary search in the range $[\min_{s}v(s),\max_{s}v(s)]$ ( see appendix \ref{app:pvarianceSection} for proofs).
 

\begin{table}
  \caption{$p$-variance}
  \label{tb:kappa}
  \centering
  \begin{tabular}{lll}
    \toprule                   
    $x$     & $\kappa_x(v)$     & Remark \\
    \midrule
    $p$ & $\min_{\omega\in\mathbb{R}}\lVert v-\omega\mathbf{1}\lVert_p  $  &  Binary search     \\&\\
    $\infty$     & $\frac{\max_{s}v(s) - \min_{s}v(s)}{2}$ & Semi-norm      \\&\\
    $2$     & $\sqrt{\sum_{s}\bigm(v(s) -\frac{\sum_{s}v(s)}{S}\bigm)^2}$      & Variance  \\&\\
    $1$     &$ \sum_{i=1}^{\lfloor (S+1)/2\rfloor}v(s_i) $ & Top half - lower half  \\ & $\quad -\sum_{i =\lceil (S+1)/2\rceil}^{S}v(s_i)$     \\
    \bottomrule
  \end{tabular}
  \begin{tabular}{l}
      where $v$ is sorted, i.e. $v(s_i)\geq v(s_{i+1}) \quad \forall i.$
  \end{tabular}      
\end{table}

\subsection{\texttt{(Sa)}-rectangular Lp robust Markov Decision Processes}\label{main:sec:saLp}
In accordance with \cite{derman2021twice}, we define $\texttt{sa}$-rectangular $L_p$ constrained uncertainty set $\mathcal{U}^{\texttt{sa}}_p$ as  
\[\mathcal{U}^{\texttt{sa}}_p := (P_0 +  \mathcal{P})\times(R_0 +  \mathcal{R})\]
where $\mathcal{P}$, $\mathcal{R}$ are noise sets around nominal kernel $P_0$ and nominal reward $R_0$ respectively. Furthermore, noise sets are \texttt{sa}-rectangular, that is
\[\mathcal{P} = \times_{s\in\mathcal{S},a\in\mathcal{A}}\mathcal{P}_{s,a},\quad\text{and}\quad  \mathcal{R} = \times_{s\in\mathcal{S},a\in\mathcal{A}}\mathcal{R}_{s,a},\]
and each component are bounded by $L_p$ norm that is
\begin{align*}
\mathcal{R}_{s,a} &=  \Bigm\{R_{s,a}\in\mathbb{R}\Bigm| \lvert R_{s,a}\rvert\leq \alpha_{s,a}\Bigm\},\quad\text{and}\\
\mathcal{P}_{s,a} &= \{P_{s,a}:\mathcal{S}\to\mathbb{R}\Bigm| \underbrace{\sum_{s'}P_{s,a}(s')=0}_{\text{ simplex condition }},  \lVert P_{s,a}\rVert_p\leq \beta_{s,a}\}
\end{align*}
with  radius vector $\alpha $ and $ \beta$. Radius vector $\beta$ is chosen small enough so that all the transition kernels in  $(P_0 +\mathcal{P})$ are well defined. Further, all transition kernels in $(P_0 + \mathcal{P})$ must have the sum of each row equal to one, with $P_0$ being a valid transition kernel satisfying this requirement. This implies that the elements of $\mathcal{P}$ must have a sum of zero across each row as ensured by simplex condition above.

Our setting differs from \cite{derman2021twice} as they didn't impose this simplex condition on the kernel noise, which renders their setting unrealistic as not all transition kernels in their uncertainty set satisfy the properties of transition kernels. This makes our reward regularizer dependent on the $q$-variance of the value function $\kappa_q(v)$, instead of the $q$-th norm of value function $\lVert v\rVert_q$ in \cite{derman2021twice}.


The main result of this subsection below states that robust Bellman operators can be evaluated using only nominal values and regularizers.
 \begin{theorem}\label{rs:saLprvi} \texttt{sa}-rectangular $L_p$ robust Bellman operators are equivalent to reward-value regularized (non-robust) Bellman operators:
\begin{align*}
    (\mathcal{T}^\pi_{\mathcal{U}^{\texttt{sa}}_p} v)(s)  =& \sum_{a}\pi(a|s)\Bigm[  -\alpha_{s,a} -\gamma\beta_{s,a}\kappa_q(v)  +R_0(s,a) +\gamma \sum_{s'}P_0(s'|s,a)v(s')\Bigm], \quad \text{and}\\
    (\mathcal{T}^*_{\mathcal{U}^{\texttt{sa}}_p} v)(s)  =& \max_{a\in\mathcal{A}}\Bigm[  -\alpha_{s,a} -\gamma\beta_{s,a}\kappa_q(v)  +R_0(s,a) +\gamma \sum_{s'}P_0(s'|s,a)v(s')\Bigm].
\end{align*}
\end{theorem}
\begin{proof}
The proof in appendix, it mainly consists of two parts: a) Separating the noise from nominal values. b) The reward noise to yields the term $-\alpha_{s,a}$ and noise in kernel yields $ -\gamma\beta_{s,a}\kappa_q(v)$. 
\end{proof}

Note, the reward penalty is proportional to both the uncertainty radiuses and a novel variance function $\kappa_p(v)$. 

We recover non-robust value iteration  by putting uncertainty radiuses (i.e. $\alpha_{s,a},\beta_{s,a}$) to zero, in the above results. Furthermore, the same is true for all subsequent robust results in this paper. 
\subsubsection*{Q-Learning}
The above result immediately implies the robust value iteration, and also suggests the Q-value iteration of the following form 
\begin{align*}
    Q_{n+1}(s,a)& = \max_{a}\Bigm[R_0(s,a)-\alpha_{s,a}-\gamma\beta_{s,a}\kappa_q(v_n) + \sum_{s'}P_0(s'|s,a)\max_{a}Q_n(s',a')\Bigm],
\end{align*}
where $  v_{n}(s) = \max_{a}Q_{n}(s,a)$, which  is further discussed in appendix \ref{app:SALpQL}.

Observe that value-variance $\kappa_p(v)$  can be estimated online, using batches or other more sophisticated methods.  This paves the path for generalizing to a model-free setting similar to \cite{Rcontamination}. 

\subsubsection*{Forbidden Transitions}
Now, we focus on the cases where $P_0(s'|s, a) =0$ for some states $s'$, that is, forbidden transitions. In many practical situations, for a given state, many transitions are impossible. For example, consider a grid world example where only a single-step jumps (left, right, up, down) are allowed, so in this case, the probability of making a multi-step jump is impossible. So upon adding noise to the kernel, the system should not start making impossible transitions. Therefore, noise set $\mathcal{P}$ must satisfy additional constraint: For any $(s,a)$ if $P_0(s'|s,a)=0$ then
   \[P(s'|s,a) = 0  , \quad \forall P \in \mathcal{P}.  \]
Incorporating this constraint without much change in the theory is one of our novel contribution, and is discussed in the appendix \ref{app:ForbiddenTransition}.

\begin{table*}
\caption{Optimal robust Bellman operator evaluation}
  \label{tb:val}
  \centering
  \begin{tabular}{lll}
    \toprule                   
    $\mathcal{U}$     & $(\mathcal{T}^*_\mathcal{U}v)(s)$     & remark \\
    \midrule
    $\mathcal{U}^{\texttt{s}}_p$ & $\min x \quad s.t. \quad \Bigm\lVert \bigm(Q_s - x\mathbf{1}\bigm)\circ \mathbf{1}\bigm(Q_s\geq x\bigm)\Bigm\rVert_p = \sigma_q(v,s)$   &Solve by binary search    \\& \\
    $\mathcal{U}^{\texttt{s}}_1$    & $\max_{k} \frac{\sum_{i=1}^kQ(s,a_i) -\sigma_\infty(v,s)}{k}$  & Highest penalized average    \\&\\
    $\mathcal{U}^{\texttt{s}}_2$     &
     By algorithm \ref{alg:main:f2}
    & High mean and variance  \\&\\
    $\mathcal{U}^{\texttt{s}}_\infty$     &$    \max_{a\in\mathcal{A}}Q(s,a) -\sigma_1(v,s)$     & Best action  \\&\\
    $\mathcal{U}^{\texttt{sa}}_p$    &$  \max_{a\in\mathcal{A}}\Bigm[Q(s,a) - \alpha_{sa} -\gamma\beta_{sa} \kappa_q(v) \Bigm] $    & Best penalized action \\&\\
    nr     &$  \max_{a}Q(s,a)$&  Best action \\ 
   \bottomrule
  \end{tabular}
  \begin{tabular}{l}
    where nr stands for Non-Robust MDP, $ \quad Q(s,a) = R_0(s,a) + \gamma\sum_{s'}P_0(s'|s,a)v(s'), $ \\sorted Q-value: $ Q(s,a_1)\geq \cdots\geq Q(s,a_A)\quad$,  $\sigma_q(v,s)= \alpha_s +\gamma\beta_{s}\kappa_q(v)$, $Q_s = Q(s,\cdot)$,\\ and $\circ $ is Hadamard product.
  \end{tabular}
\end{table*}

\subsection{\texttt{S}-rectangular Lp robust Markov Decision Processes}\label{main:sec:sLp}
In this subsection, we discuss the core contribution of this paper: the evaluation of robust Bellman operators for the \texttt{s}-rectangular uncertainty set. 

We begin by defining \texttt{s}-rectangular $L_p$ constrained uncertainty set $\mathcal{U}^{\texttt{s}}_p$ as  
\[\mathcal{U}^{\texttt{s}}_p := (P_0 +  \mathcal{P})\times(R_0 +  \mathcal{R})\]
where noise sets are \texttt{s}-rectangular, 
\[\mathcal{P} = \times_{s\in\mathcal{S}}\mathcal{P}_{s},\quad\text{and}\quad  \mathcal{R} = \times_{s\in\mathcal{S}}\mathcal{R}_{s},\]
and each component are bounded by $L_p$ norm,
\begin{align*}
        \mathcal{R}_{s} &=  \Bigm\{R_s:\mathcal{A}\to\mathbb{R}\Bigm| \lVert R_s\rVert_p\leq \alpha_{s}\Bigm\},\quad\text{and}\\
     \mathcal{P}_{s} &= \Bigm\{P_s:\mathcal{S}\times\mathcal{A}\to\mathbb{R}\Bigm|\lVert P_s\rVert_p\leq \beta_{s}, \sum_{s'}P_s(s',a)=0,\forall a
     \Bigm\},
\end{align*}
with  radius vectors  $\alpha$ and  small enough $\beta$. 



 The result below shows that, compared to the \texttt{sa}-rectangular case, the policy evaluation for the \texttt{s}-rectangular case has an extra dependence on the policy.
\begin{theorem} \label{rs:SLpPlanning}(Policy Evaluation) \texttt{S}-rectangular $L_p$ robust Bellman operator is equivalent to reward-value-policy regularized (non-robust) Bellman operator:
\begin{align*}
    &(\mathcal{T}^\pi_{\mathcal{U}^{\texttt{s}}_p} v)(s)  =   -\bigm[\alpha_s +\gamma\beta_{s}\kappa_q(v)\bigm]\lVert\pi_s\rVert_q +\sum_{a}\pi(a|s)\Bigm[R_0(s,a) +\gamma \sum_{s'}P_0(s'|s,a)v(s')\Bigm],
\end{align*}
where $\lVert \pi_s\rVert _q$ is $q$-norm of the vector $\pi(\cdot|s)\in\Delta_{\mathcal{A}}$.
\end{theorem}
\begin{proof}
The proof in the appendix: the techniques are similar to as its $\texttt{sa}$-rectangular counterpart.
\end{proof}

The reward penalty in this case has an additional dependence on the norm of the policy ($\lVert \pi_s\rVert_q$). This norm is conceptually similar to entropy regularization $\sum_{a}\pi(a|s)\ln(\pi(a|s))$, which is widely studied in the literature \cite{EntReg1,EntReg2,EntReg3,SoftQL,TRPO}, and other regularizers such as $\sum_{a}\pi(a|s)\textit{tsallis}(\frac{1-\pi(a|s)}{2})$, $\sum_{a}\pi(a|s)cos(cos(\frac{\pi(a|s) }{2}))$, etc. 

\textbf{Note:} These regularizers, which are convex functions, are often used to promote stochasticity in the policy and thus improve exploration during learning. However, the above result shows another benefit of these regularizers: they can improve robustness, which in turn can lead to better generalization. 


 
 In literature, the above regularizers are scaled with arbitrary chosen constant, here we have the  different constant $\alpha_s +\gamma\beta_{s}\kappa_q(v)$ for different states.

 This extra dependence makes the policy improvement a more challenging task and thus, presents a richer theory.  

\begin{theorem}\label{rs:rve}(Policy improvement)
For any vector $v$ and state $s$,  $(\mathcal{T}^*_{\mathcal{U}^{\texttt{s}}_p}v)(s)$ is the minimum value of $x$ that satisfies 
\begin{align}\label{eq:rve:val}
    \Bigm[\sum_{a}\Bigm(Q(s,a) - x\Bigm)^p\mathbf{1}\Bigm( Q(s,a) \geq x\Bigm)\Bigm]^{\frac{1}{p}}  = \sigma,
\end{align}
where   $Q(s,a) = R_0(s,a) + \gamma\sum_{s'} P_0(s'|s,a)v(s')$, and $\sigma = \alpha_s + \gamma\beta_s\kappa_q(v)$.
\end{theorem}
\begin{proof} The proof is in the appendix; the main steps are:
\begin{align*}
 &(\mathcal{T}^*_{\mathcal{U}^{\texttt{s}}_p}v)(s)= \max_{\pi}(\mathcal{T}^\pi_{\mathcal{U}^{\texttt{s}}_p} v)(s),  \qquad \text{(from definition)}\\
 &\text{( Using policy evaluation Theorem \ref{rs:SLpPlanning})}\\
 &=  \max_{\pi}\Bigm[(\mathcal{T}^\pi_{(P_0,R_0)}v)(s)-\bigm[\alpha_s +\gamma\beta_{s}\kappa_q(v)\bigm]\lVert\pi_s\rVert_q \Bigm]\\
 &= \max_{\pi_s\in\Delta_\mathcal{A}}  \langle \pi_s,Q_s\rangle- \sigma\lVert \pi_s\rVert_q  \quad \text{( where $Q_s = Q(\cdot|s)$).}
\end{align*}
The solution to the above optimization problem is technically complex. Specifically, for $p=2$, the solution is known as the water filling/pouring lemma \cite{anava2016k}, we generalize it to the $L_p$ case, in the appendix.
\end{proof}
 To better understand the nature of \eqref{eq:rve:val}, lets look at the 'sub-optimality distance' function $g$,
 \[g(x) :=\Bigm[\sum_{a}\Bigm(Q(s,a) - x\Bigm)^p\mathbf{1}\Bigm( Q(s,a) \geq x\Bigm)\Bigm]^{\frac{1}{p}}. \]
 
The $g(x)$ is the cumulative difference between $x$ and  the Q-values of actions whose Q-value is greater than $x$. The function is monotonically decreasing, with a lower bound of $\sigma$ at $x=\max_{a}Q(s,a)-\sigma$ and a value of zero for all $x\geq \max_{a}Q(s,a)$. Since,  $(\mathcal{T}^*_{\mathcal{U}^{\texttt{s}}_p}v)(s)$ is the value of $x$ at which the "sub-optimality distance" $g(x)$ is equal to the "uncertainty penalty" $\sigma$. Hence, \eqref{eq:rve:val} can be approximately solved using a binary search between the interval $[\max_{a}Q(s,a)-\sigma, \quad \max_{a}Q(s,a)]$.

We invite the readers to consider the dependence of $(\mathcal{T}^*_{\mathcal{U}^{\texttt{s}}_p}v)(s)$ on $p,\alpha_s,$ and $\beta_s$, specifically:
\begin{enumerate}
    \item If $\alpha_s=\beta_s =0$ then $\sigma =0$ which implies $(\mathcal{T}^*_{\mathcal{U}^{\texttt{s}}_p}v)(s)=\max_{a}Q(s,a)$, same as non-robust case.
    \item If $p=\infty$ then $(\mathcal{T}^*_{\mathcal{U}^{\texttt{s}}_p}v)(s)=\max_{a}Q(s,a) -\sigma$, as in the \texttt{sa}-rectangular case.
    \item For $p=1,2$, \eqref{eq:rve:val} becomes linear and quadratic equation respectively, hence can be solved exactly.
    \item As $\alpha_s$ and $\beta_s$ increase, $\sigma$ increases, resulting in a decrease in $(\mathcal{T}^*_{\mathcal{U}^{\texttt{s}}_p}v)(s)$ at a rate that becomes smaller as $\sigma$ increases. When $\sigma$ is sufficiently small, $(\mathcal{T}^*_{\mathcal{U}^{\texttt{s}}_p}v)(s) = \max_{a}Q(s,a)-\sigma$.
\end{enumerate}
Solution to \eqref{eq:rve:val} can be obtained in closed form for the cases of $p=1,\infty$, exactly by algorithm \ref{alg:main:f2} for $p=2$, and approximately by binary search for general $p$, as summarized in table \ref{tb:val}. 

In this section, we have demonstrated that robust Bellman operators can be efficiently evaluated for both $\texttt{sa}$ and \texttt{s} rectangular $L_p$ robust MDPs, thus enabling efficient robust value iteration. In the following sections, we discuss the nature of optimal policies and the time complexity of robust value iteration. Finally, we present experiments validating the time complexity of robust value iteration.
\begin{algorithm}[tb]
\caption{\texttt{s}-rectangular $L_2$ robust Bellman operator\\ (see algorithm 1 of \cite{anava2016k}}
\label{alg:main:f2}
\textbf{Input}: $v, s,  x = Q(s,\cdot), $ and $\sigma = \alpha_s +\gamma\beta_{s}\kappa_q(v)$\\
\textbf{Output}: $(\mathcal{T}^*_{\mathcal{U}^{\texttt{s}}_p}v)(s)$
\begin{algorithmic}[1] 
\STATE Sort $x$ such that $x_1\geq x_2, \cdots \geq x_A$.
\STATE Set $k=0$ and $\lambda = x_1-\sigma$

\WHILE{$k \leq A-1  $ and $\lambda \leq x_k$}
    \STATE  $k = k+1$
    \STATE  \[\lambda = \frac{1}{k}\Bigm[\sum_{i=1}^{k}x_i - \sqrt{k\sigma^2 + (\sum_{i=1}^{k}x_i^2 - k\sum_{i=1}^{k}x_i)^2}\Bigm]\]
\ENDWHILE

\STATE \textbf{return} $\lambda$
\end{algorithmic}
\end{algorithm}

\begin{table*}
\caption{Optimal Policy}
  \label{tb:Opt:pi}
  \centering
  \begin{tabular}{lll}
    \toprule                   
    $\mathcal{U}$     & $\pi^*_\mathcal{U}(a|s) \propto$      & Remark \\
    \midrule
    $\mathcal{U}^{\texttt{s}}_p$ &$ A(s,a)^{p-1}\mathbf{1}(A(s,a) \geq 0)$ &  Top actions proportional to\\ & &  $(p-1)$-th power of advantage\\
    &\\
    $\mathcal{U}^s_1$    &$\mathbf{1}(A(s,a) \geq 0)$&Top  actions with uniform probability  \\
    &\\
    $\mathcal{U}^s_2$     &$A(s,a)\mathbf{1}(A(s,a) \geq 0)$ & Top   actions proportion to advantage \\ 
    &\\
    $\mathcal{U}^s_\infty$     &$  \mathbf{1}(A(s,a)=0)$&  Best action \\ &\\
    $\mathcal{U}^{\texttt{sa}}_p$    &$\mathbf{1}(A(s,a)=\max_{a}A(s,a))  $&  Best regularized action\\& \\
    $(P_0,R_0)$     &$  \mathbf{1}(A(s,a)=0)$&  Non-robust MDP: Best action \\ 
    \bottomrule
 \end{tabular}
  \begin{tabular}{l}
  where $Q(s,a) = R_0(s,a) + \gamma\sum_{s'}P_0(s'|s,a)v^*_\mathcal{U}(s')$, and $A(s,a) = Q(s,a)-v^*_\mathcal{U}(s)$.
  \end{tabular}
\end{table*}
\section{Optimal Policies}
In the previous sections, we discussed how to efficiently obtain the optimal robust value functions. This section focuses on utilizing these optimal robust value functions to derive the optimal robust policy using 
\begin{align*}
    &\pi^*_{\mathcal{U}} \in \text{arg}\max_{\pi}\mathcal{T}^\pi_{\mathcal{U}} v^*_{\mathcal{U}}.
\end{align*}
This implies, the robust optimal policy $\pi^*_{\mathcal{U}}(\cdot|s)$ at state $s$, is the policy $\pi$ that maximizes   
\[\sum_{a}\pi(a|s)\min_{(P,R)\in\mathcal{U}}\Bigm[R(s,a) +\gamma \sum_{s'}P(s'|s,a)v^*_{\mathcal{U}}(s')\Bigm]. \]

 \textbf{Non-robust MDP} admits a deterministic optimal policy that maximizes the optimal Q-value $Q(s,a) := R(s,a) +\gamma \sum_{s'}P(s'|s,a)v^*_{(P,R)}(s')$. 
 
 \textbf{\texttt{sa}-rectangular robust MDPs} are known to admit a deterministic optimal robust policy \cite{Iyenger2005,Nilim2005RobustCO}. Moreover, from Theorem \ref{rs:saLprvi}, it clear that a \texttt{sa}-rectangular $L_p$ robust MDP has a deterministic optimal robust policy that maximizes the regularized Q-value $ Q(s,a) = -\alpha_{s,a} -\gamma\beta_{s,a}\kappa_q(v)  +R_0(s,a) +\gamma \sum_{s'}P_0(s'|s,a)v^*_{\mathcal{U}^{\texttt{sa}}_p}(s')$.

 \textbf{\texttt{s}-rectangular robust MDPs}: 
For this case, it is known that all optimal robust policies can be stochastic \cite{RVI}, however, it was not previously known what the nature of this stochasticity was. The result below provides the first explicit characterization of robust optimal policies.

\begin{theorem}\label{rs:srect:optimalPolicy} The  optimal robust policy $\pi^*_{\mathcal{U}^{\texttt{s}}_p}$ can be computed using optimal robust value function as:
\begin{align*}
    \pi^*_{\mathcal{U}^{\texttt{s}}_p}(a|s) \propto [ Q(s,a)-v^*_{\mathcal{U}^{\texttt{s}}_p}(s) ]^{p-1}\mathbf{1} \bigm( Q(s,a)\geq  v^*_{\mathcal{U}^{\texttt{s}}_p}(s)\bigm)
\end{align*}
where $Q(s,a) = R_0(s,a) + \gamma\sum_{s'}P_0(s'|s,a) v^*_{\mathcal{U}^{\texttt{s}}_p}(s)$.
\end{theorem}

The above policy is a threshold policy that takes actions with a positive advantage, which is proportional to the advantage function, while giving more weight to actions with higher advantages and avoiding playing actions that are not useful. This policy is different from the optimal policy in soft-Q learning with entropy regularization, which is a softmax policy of the form $\pi(a|s) \propto e^{\eta (Q(a|s)-v(s))}$ \cite{SoftQL,EntReg1,TRPO}. To the best of our knowledge, this type of policy has not been presented in literature before.


The special cases of the above theorem for $p=1,2,\infty$ along with others are summarized in table \ref{tb:Opt:pi}.

\begin{algorithm}[tb]
\caption{Online \texttt{s}-rectangular $L_p$ robust value iteration}
\label{main:alg:SLp}
\textbf{Input}: Initialize $Q,v$ randomly, $s_0\sim \mu,$ and $n=0$.\\
\textbf{Output}: $v = v^*_{\mathcal{U}^\texttt{s}_p}$.
\begin{algorithmic}[1] 
\WHILE{ not converged; $n = n+1$}
    \STATE Estimate $\kappa_p(v)$ using table \ref{tb:kappa}.
    \STATE Approximate $(\mathcal{T}^*_{\mathcal{U}^\texttt{s}_p}v)(s_n)$ using table \ref{tb:val} and update
    \[v(s_n) =v(s_n) + \eta_n[(\mathcal{T}^*_{\mathcal{U}^\texttt{s}_p}v)(s_n)-v(s_n)] . \]
    \STATE Play action $a_n=a$ with probability proportional to
    \[ [Q(s_n,a)-v(s_n)]^{p-1}\mathbf{1}(Q(s_n,a) \geq v(s_n)),\] 
    and get next state $s_{n+1}$ from the environment.
    \STATE Update Q-value:
    \begin{align*}
        Q(s_n,a_n) = &Q(s_n,a_n) + \eta'_n[R(s_n,a_n) +\gamma v(s_{n+1})-Q(s_n,a_n)].
    \end{align*}
\ENDWHILE
\end{algorithmic}
\end{algorithm}



\begin{table*} 
  \caption{Relative running cost (time) for value iteration}
  \centering
  \begin{tabular}{lllllllllllll}
    \toprule                   
    S&A& nr& $\mathcal{U}^{sa}_1$ LP&$\mathcal{U}^{s}_1$ LP&$\mathcal{U}^{sa}_1$&$\mathcal{U}^{sa}_2$ &$\mathcal{U}^{sa}_\infty$&$\mathcal{U}^{s}_1$&$\mathcal{U}^{s}_2$&$\mathcal{U}^{s}_\infty$&$\mathcal{U}^{sa}_{10}$&$\mathcal{U}^{s}_{10}$\\
     \midrule
    10&10&1&1438&72625&1.7&1.5&1.5&1.4&2.6&1.4&5.5&33\\
   30&10&1&6616&629890&1.3&1.4&1.4&1.5&2.8&3.0&5.2&78\\   
    50&10&1&6622&4904004&1.5&1.9&1.3&1.2&2.4&2.2&4.1&41\\
    100&20&1&16714&NA&1.4&1.5&1.5&1.1&2.1&1.5&3.2&41\\
    \bottomrule
  \end{tabular}
  \begin{tabular}{l}
       nr stands for Non-robust MDP
  \end{tabular}
  \end{table*}

\section{Time complexity }
In this section, we examine the time complexity of robust value iteration:
\[ v_{n+1} := \mathcal{T}^*_\mathcal{U}v_n\]
for different $L_p$ robust MDPs assuming the knowledge of nominal values ($P_0$, $R_0$). Since, the optimal robust Bellman operator $\mathcal{T}^*_\mathcal{U}$ is $\gamma$-contraction operator \cite{RVI}, meaning that it requires only $O(\log(\frac{1}{\epsilon}))$ iterations to obtain an $\epsilon$-close approximation of the optimal robust value. The main challenge is to calculate the cost of one iteration.



The evaluation of the optimal robust Bellman operators in Theorem \ref{rs:saLprvi} and Theorem \ref{rs:rve} has three main components. A) Computing $\kappa_p(v)$, which can be done differently depending on the value of $p$, as shown in table \ref{tb:kappa}. B) Computing the Q-value from $v$, which requires $O(S^2A)$ in all cases. And finally, C) Evaluating optimal robust Bellman operators from Q-values, which requires different operations such as sorting of the Q-value, calculating the best action, and performing a binary search, etc., as shown in table \ref{tb:val}. The overall complexity of the evaluation is presented in table \ref{tb:time}, with the proofs provided in appendix \ref{app:timeComplexitySection}.

We can observe that when the state space $S$ is large, the complexity of the robust MDPs is the same as that of the non-robust MDPs, as the complexity of all robust MDPs is the same as non-robust MDPs at the limit $S\to \infty$ (keeping action space $A$ and tolerance $\epsilon$ constant). This is verified by our experiments, thus concluding that the $L_p$ robust MDPs are as easy as non-robust MDPs.

\begin{table}
  \caption{Time complexity}
  \label{tb:time}
  \centering
  \begin{tabular}{ll}
    \toprule            
    & Total cost $O$ \\ 
   \midrule
 Non-Robust MDP & $\log(1/\epsilon)S^2A$\\
 $\mathcal{U}^{\mathtt{sa}}_1$& $\log(1/\epsilon)S^2A$\\
 $\mathcal{U}^{\mathtt{sa}}_2$ &$\log(1/\epsilon)S^2A$\\
 $\mathcal{U}^{\mathtt{sa}}_\infty$&$\log(1/\epsilon)S^2A$\\
  $\mathcal{U}^{\texttt{s}}_1$& $\log(1/\epsilon)(S^2A + SA\log(A))$\\
 $\mathcal{U}^{\texttt{s}}_2$&$\log(1/\epsilon)(S^2A + SA\log(A))$\\
 $\mathcal{U}^{\texttt{s}}_\infty$ &$\log(1/\epsilon)S^2A$\\
 \midrule
 $\mathcal{U}^{\mathtt{sa}}_p$& $\log(1/\epsilon)\bigm(S^2A + S\log(S
 /\epsilon)\bigm) $\\
 $\mathcal{U}^{\texttt{s}}_p$&$\log(1/\epsilon)\bigm( S^2A+ SA\log(A/\epsilon) \bigm) $\\
 \bottomrule
 Convex $\mathcal{U}$ & Strongly NP Hard \\
 \bottomrule
 \end{tabular}
\end{table}

\section{Experiments}\label{app:experiments}
In this section, we present numerical results that demonstrate the effectiveness of our methods, verifying our theoretical claims.

Table 4 and Figure \ref{fig:asym} demonstrate the relative cost (time) of robust value iteration compared to non-robust MDP, for randomly generated kernel and reward functions with varying numbers of states $S$ and actions $A$. The results show that \texttt{s} and \texttt{sa}-rectangular MDPs are indeed costly to solve using numerical methods such as Linear Programming (LP). Our methods perform similarly to non-robust MDPs, especially for $p=1,2,\infty$. For general $p$, binary search is required for acceptable tolerance, which requires $30-50$ iterations, leading to a little longer computation time.
\begin{figure}
      \centering
       \includegraphics[width=\linewidth]{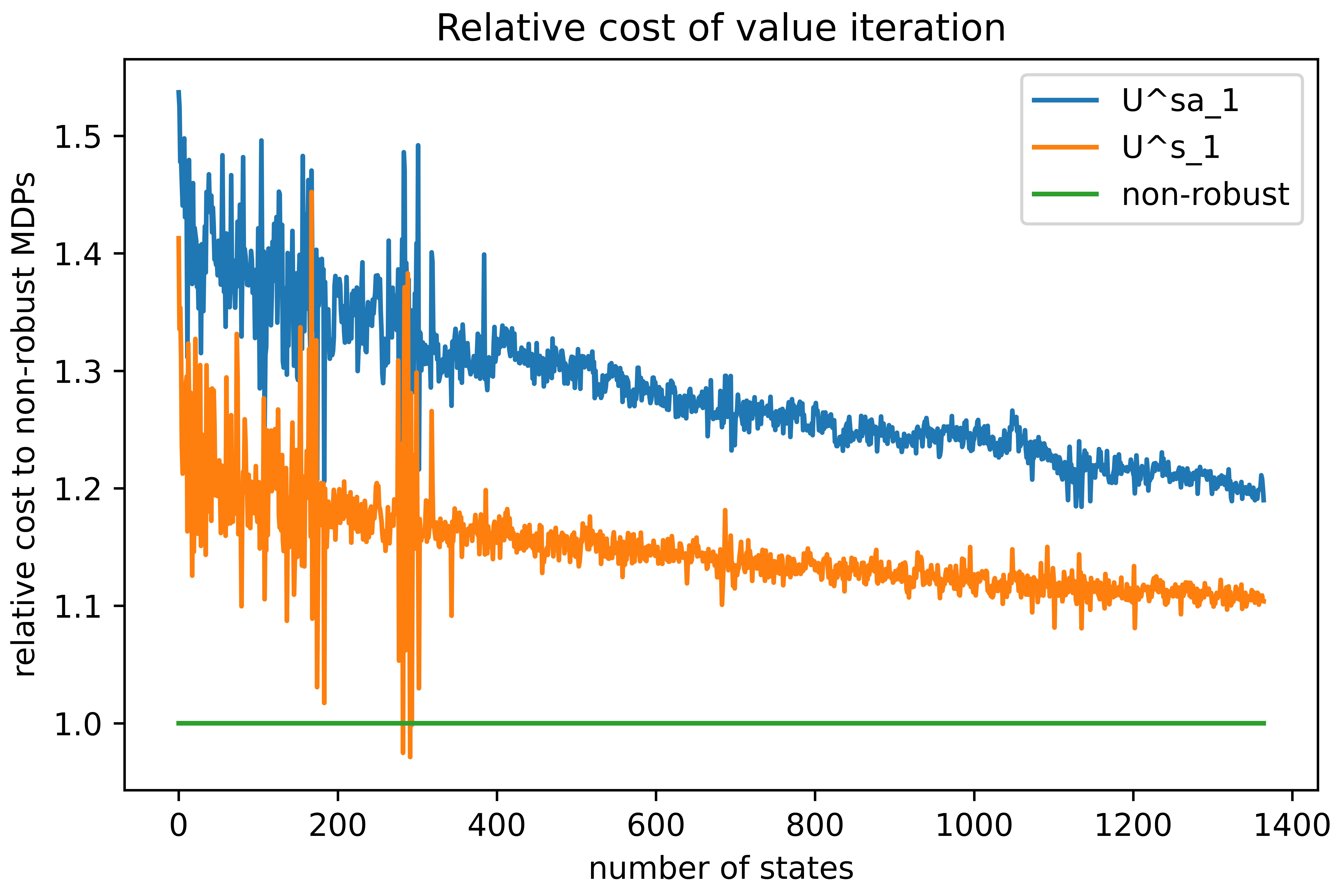}
      \caption{Relative cost of value iteration w.r.t. non-robust MDP at different $S$ with fixed $A=10$.}
      \label{fig:asym}
  \end{figure}


As our complexity analysis shows, value iteration's relative cost converges to 1 as the number of states increases while keeping the number of actions fixed. This is confirmed by Figure \ref{fig:asym}.

The rate of convergence for all the settings tested was the same as that of the non-robust setting, as predicted by the theory. The experiments ran a few times, resulting in some stochasticity in the results, but the trend is clear. Further details can be found in section \ref{app:experiments}. 



\section{Conclusion and future work}\label{conclusion}

We present an efficient robust value iteration for \texttt{s}-rectangular $L_p$-robust MDPs. Our method can be easily adapted to an online setting, as shown in Algorithm \ref{main:alg:SLp} for \texttt{s}-rectangular $L_p$-robust MDPs.
Algorithm \ref{main:alg:SLp} is a two-time-scale algorithm, where the Q-values are approximated at a faster time scale and the value function is approximated from the Q-values at a slower time scale. The $p$-variance function $\kappa_p$ can be estimated in an online fashion using batches or other sophisticated methods. The convergence of the algorithm can be guaranteed from \cite{borkarBook}; however, its analysis is left for future work. 

Additionally, we introduce a novel value regularizer ($\kappa_p$) and a novel threshold policy which may help to obtain more robust and generalizable policies.

Further research could focus on other types of uncertainty sets, potentially resulting in different kinds of regularizers and optimal policies. 

\bibliography{main}
\bibliographystyle{plain}

\appendix

\section*{How to read appendix}
\begin{enumerate}

    \item Section \ref{app:relatedWork} contains related work.
    \item Section \ref{app:properties} contains additional properties and results that couldn't be included in the main section for the sake of clarity and space. Many of the results in the main paper is special cases of the results in this section. 
    \item Section \ref{app:ForbiddenTransition} contains the discussion on zero transition kernel (forbidden transitions).
    \item Section \ref{app:UCRL} contains a possible connection this work to UCRL.
    \item Section \ref{app:experiments} contains additional experimental results and a detailed discussion.
     \item All the proofs of the main body of the paper is presented in the section  \ref{app:srLp} and \ref{app:timeComplexitySection}.
    \item Section \ref{app:pvarianceSection} contains helper results for section \ref{app:srLp}. Particularly, it discusses $p$-mean function $\omega_p$ and $p$-variance function $\kappa_p$.
    \item Section \ref{app:waterPouringSection} contains helper results for section \ref{app:srLp}. Particularly, it discusses $L_p$ water pouring lemma, necessary to evaluate robust optimal Bellman operator (learning) for $\mathtt{s}$-rectangular $L_p$ robust MDPs.
    \item Section \ref{app:timeComplexitySection} contains time complexity proof for model based algorithms.
    \item Section \ref{app:SALpQL} develops Q-learning machinery for $(\mathtt{sa})$-rectangular $L_p$ robust MDPs based on the results in the main section. It is not used in the main body or anywhere else, but this provides a good understanding for algorithms proposed in section \ref{app:ModelBasedAlgorithms} for $(\mathtt{sa})$-rectangular case.
    \item Section \ref{app:ModelBasedAlgorithms} contains model-based algorithms for $\mathtt{s}$ and $(\mathtt{sa})$-rectangular $L_p$ robust MDPs. It also contains, remarks for special cases for $p=1,2,\infty$.
\end{enumerate}
\section{Related Work} \label{app:relatedWork}
\subsection*{ R-Contamination Uncertainty Robust MDPs}
The paper \cite{Rcontamination} considers the following uncertainty set for some fixed constant $0 \leq R \leq 1$,
\begin{align}
    \mathcal{P}_{sa} = \{ (1-R)(P_0)(\cdot|s,a) + RP \mid P \in \Delta_{\mathcal{S}} \}, \quad s\in \mathcal{S}, a\in\mathcal{A},
\end{align}
and  $\mathcal{P} = \otimes_{s,a}\mathcal{P}_{s,a}, \qquad \mathcal{U} = \{R_0\}\times\mathcal{P}$. The robust value function $v^\pi_\mathcal{U}$ is the fixed point of the robust Bellman operator defined as 
\begin{align}
    (\mathcal{T}^\pi_\mathcal{U}v)(s) :=& \min_{P\in\mathcal{P}}\sum_{a}\pi(a|s)[R_0(s,a) + \gamma \sum_{s'}P(s'|s,a)v(s')],\\
    =&\sum_{a}\pi(a|s)[R_0(s,a) - \gamma R\max_{s}v(s) + (1-R)\gamma \sum_{s'}P_0(s'|s,a)v(s')].
\end{align}
And the optimal robust value function $v^*_\mathcal{Ua }$ is the fixed point of the optimal robust Bellman operator defined as 
\begin{align}
    (\mathcal{T}^*_\mathcal{U}v)(s) :=& \max_{\pi}\min_{P\in\mathcal{P}}\sum_ {a}\pi(a|s)[R_0(s,a) + \gamma (1-R)\sum_{s'}P(s'|s,a)v(s')],\\
    =&\max_{a}[R_0(s,a) - \gamma R\max_{s}v(s) + \gamma (1-R)\sum_{s'}P_0(s'|s,a)v(s')].
\end{align}
Since, the uncertainty set is \texttt{sa}-rectangular, hence the map is a contraction \cite{Nilim2005RobustCO}, so the robust value iteration here, will also converge linearly similar to non-robust MDPs. It is also possible to obtain Q-learning as following
\begin{align}
    Q_{n+1}(s,a) = R_0(s,a)  - \gamma R\max_{s,a}Q_n(s,a) + \gamma(1-R) \sum_{s'}P_0(s'|s,a)\max_{s'}Q_n(s',a').
\end{align}
Convergence of the above Q-learning follows from the contraction of robust value iteration. Further, it is easy to see that model-free Q-learning can be obtained from the above.

A follow-up work \cite{PG_RContamination} proposes a policy gradient method for the same.
\begin{proposition}(Theorem 3.3 of \cite{PG_RContamination}) Consider a class of policies $\Pi$ satisfying Assumption 3.2 of \cite{PG_RContamination}. The gradient of the robust return is given by
\begin{align*}
    \nabla \rho^{\pi_\theta} = &\frac{\gamma R}{(1-\gamma)(1-\gamma +\gamma R)}\sum_{s,a}d^{\pi_\theta}_\mu(s,a)\nabla{\pi_\theta}(a|s)Q^{\pi_\theta}_{\mathcal{U}}(s,a)  \\
    &\qquad +  \frac{1}{1-\gamma +\gamma R}\sum_{s,a}d^{\pi_\theta}_{s_\theta}(s,a)\nabla{\pi_\theta}(a|s)Q^{\pi_\theta}_{\mathcal{U}}(s,a) ,
\end{align*}
where $s_\theta\in \text{arg}\max v^{\pi_\theta}_{\mathcal{U}}(s) $, and $Q^\pi_\mathcal{U}(s,a) = \sum_{a}\pi(a|s)\bigm[R_0(s,a) - \gamma R \max_{s}v^\pi_\mathcal{U}(s) + \gamma (1-R)\sum_{s'}P_0(s'|s,a)v^\pi_\mathcal{U}(s')\bigm].$
\end{proposition}
The work shows that the proposed robust policy gradient method converges to the global optimum asymptotically under direct policy parameterization.

The uncertainty set considered here, is \texttt{sa}-rectangular, as uncertainty in each state-action is independent, hence the regularizer term $(\gamma R\max_{s}v(s))$ is independent of policy, and the optimal (and greedy) policy is deterministic. It is unclear, how the uncertainty set can be generalized to the $s$-rectangular case. Observe that the above results resemble very closely our \texttt{sa}-rectangular $L_1$ robust MDPs results.

\subsection*{Twice Regularized MDPs}
The paper \cite{derman2021twice} converts robust MDPs to twice regularized MDPs, and proposes a gradient based policy iteration method for solving them. 

\begin{proposition} (corollary 3.1 of \cite{derman2021twice}) (\texttt{s}-rectangular reward robust policy evaluation) Let the uncertainty set be $\mathcal{U} = (R_0+\mathcal{R})\times\{P_0\}$, where $\mathcal{R}_s =\{ r_s\in\mathbf{R}^{\mathcal{A}}\mid \lVert r_s\rVert \leq \alpha_{s}\}$ for all $s\in\mathcal{S}$. Then the robust value function $v^\pi_{\mathcal{U}}$ is the optimal solution to the convex optimization problem:
\[ \max_{v\in\mathbf{R}^\mathcal{A}}\langle \mu,v\rangle \quad s.t.\quad v(s)\leq (\mathcal{T}^\pi_{R_0,P_0}v)(s) -\alpha_{s}\lVert\pi_s\rVert, \qquad \forall s\in\mathcal{S}.\]
\end{proposition}
It derives the policy gradient for reward robust MDPs to obtain the optimal robust policy $\pi^*_\mathcal{U}$.

\begin{proposition} (Proposition 3.2 of \cite{derman2021twice}) (\texttt{s}-rectangular reward robust policy gradient) Let the uncertainty set be $\mathcal{U} = (R_0+\mathcal{R})\times\{P_0\}$, where $\mathcal{R}_s =\{ r_s\in\mathbf{R}^{\mathcal{A}}\mid \lVert r_s\rVert \leq \alpha_{s}\}$ for all $s\in\mathcal{S}$. Then the gradient of the reward robust objective $ \rho^\pi_{\mathcal{U}} := \langle \mu, v^\pi_{\mathcal{U}}\rangle$ is given by
\[ \nabla \rho^\pi_\mathcal{U} = \mathbf{E}_{(s,a)\sim d^\pi_{P_0}}\Bigm[ \nabla\ln(\pi(a|s))\bigm(Q^\pi_{\mathcal{U}}(s,a) -\alpha_{s}\frac{\pi(a|s)}{\lVert \pi_s \rVert}\bigm)\Bigm],\]
where $Q^\pi_{\mathcal{U}}(s,a):=\min_{(R,P)\in\mathcal{U}}[R(s,a)+\gamma \sum_{s'}P(s'|s,a)v^\pi_{\mathcal{U}}(s')].$
\end{proposition}

\begin{proposition} (Corollary 4.1 of \cite{derman2021twice}) (\texttt{s}-rectangular general robust policy evaluation) Let the uncertainty set be $\mathcal{U} = (R_0+\mathcal{R})\times\{P_0+\mathcal{P}\}$, where $\mathcal{R}_s =\{ r_s\in\mathbf{R}^{\mathcal{A}}\mid \lVert r_s\rVert \leq \alpha_{s}\}$  and $\mathcal{P}_s =\{ P_s\in\mathbf{R}^{\mathcal{S}\times\mathcal{A}}\mid \lVert P_s\rVert \leq \beta_{s}\}$ for all $s\in\mathcal{S}$. Then the robust value function $v^\pi_{\mathcal{U}}$ is the optimal solution to the convex optimization problem:
\[ \max_{v\in\mathbf{R}^\mathcal{A}}\langle \mu,v\rangle \quad s.t.\quad v(s)\leq (\mathcal{T}^\pi_{R_0,P_0}v)(s) -\alpha_{s}\lVert\pi_s\rVert -\gamma\beta_{s}\lVert v\rVert\lVert\pi_s\rVert, \qquad \forall s\in\mathcal{S}.\]
\end{proposition}
Same as the reward robust case, the paper tries to find a policy gradient method to obtain the optimal robust policy. Unfortunately, the dependence of regularizer terms on value makes it a very difficult task. Hence it proposes the $R^2$MPI algorithm (algorithm 1 of ~\cite{derman2021twice}) for the purpose that optimizing the greedy step via projection onto the simplex using a black box solver. Note that the above proposition is not same as our policy evaluation (although it looks similar), it requires some extra assumptions (assumption 5.1 ~\cite{derman2021twice}) and lot of work ensure $R^2$ Bellman operator is contraction etc. In our case, we directly evaluate robust Bellman operator that has already proven to be a contraction, hence we don't require any extra assumption nor any other work as ~\cite{derman2021twice}. 

Our work makes improvements over this work by explicitly solving both policy evaluation and policy improvement in general robust MDPs. It also makes more realistic assumptions on the transition kernel uncertainty set. 

\subsection*{Regularizer solves Robust MDPs}
The work \cite{MaxEntRL} looks in the opposite direction than we do. It investigates the impact of the popularly used entropy regularizer on robustness. It finds that MaxEnt can be used to maximize a lower bound on a certain robust RL objective (reward robust).  

As we noticed that $\lVert \pi_s \rVert_q$ behaves like entropy in our regularization. Further, our work also deals with uncertainty in transition kernel in addition to the uncertainty in reward function.

\subsection*{Upper Confidence RL}
The upper confidence setting in \cite{UCRL,UCRL2} is very similar to our $L_p$ robust setting. We refer to this discussion in section \ref{app:UCRL}.

\section{S-rectangular: More Properties}\label{app:properties}
 
\begin{definition} We begin with the following notational definitions.
    \begin{enumerate}
        \item  $Q$-value at value function $v$ is defined as 
        \[Q^v(s,a) := R_0(s,a) + \gamma\sum_{s'}P_0(s'|s,a) v(s').\]
        \item Optimal $Q$-value is defined as 
        \[Q^*_{\mathcal{U}}(s,a) = R_0(s,a) + \gamma\sum_{s'}P_0(s'|s,a) v^*_{\mathcal{U}}(s')\]
        \item With little abuse of notation, $Q(s,a_i)$ shall denote the $i$th best value in state $s$, that is
        \[Q(s,a_1)\geq Q(s,a_2)\geq,\cdots,\geq Q(s,a_A).\]
        \item $\pi^v_\mathcal{U}$ denotes the greedy policy at value function $v$, that is
         \[\mathcal{T}^{*}_{\mathcal{U}}v = \mathcal{T}^{\pi^v_{\mathcal{U}}}_{\mathcal{U}}v. \]
         \item $\chi_p(s)$ denotes the number of active actions in state $s$ in \texttt{s}-rectangular $L_p$ robust MDPs, defined as
        \[\chi_p(s) :=\bigm\lvert \{a \mid \pi^*_{\mathcal{U}^s_p}(a|s) \geq 0\}\bigm\rvert .\]
        \item $\chi_p(p,s)$ denotes the number of active actions in state $s$ at value function $v$ in \texttt{s}-rectangular $L_p$ robust MDPs, defined as
        \[\chi_p(v,s) :=\bigm\lvert \{a \mid \pi^v_{\mathcal{U}^s_p}(a|s) \geq 0\}\bigm\rvert .\]
    \end{enumerate}
\end{definition}
We saw above that optimal policy in \texttt{s}-rectangular robust MDPs may be stochastic. The action that has a positive advantage is active and the rest are inactive. Let $\chi_p(s)$ be the number of active actions in state $s$, defined as 
\begin{align}
    \chi_p(s) :=\bigm\lvert \{a \mid \pi^*_{\mathcal{U}^s_p}(a|s) \geq 0\}\bigm\rvert =\bigm\lvert \{a \mid Q^*_{\mathcal{U}^s_p}(s,a) \geq v^*_{\mathcal{U}^s_p}(s)\}\bigm\rvert.
\end{align}
Last equality comes from Theorem \ref{rs:srect:optimalPolicy}. One direct relation between Q-value and value function is given by 
\begin{align}
    v^*_{\mathcal{U}^s_p}(s) = \sum_{a}\pi^*_{\mathcal{U}^s_p}(a|s)\Bigm[-\bigm(\alpha_s +\gamma\beta_{s}\kappa_q(v)\bigm)\lVert\pi^*_{\mathcal{U}^s_p}(\cdot|s)\rVert_q  +Q^*_{\mathcal{U}^s_p}(s,a)\Bigm].
\end{align}
The above relation is very convoluted compared to non-robust and  \texttt{sa}-rectangular robust cases. The property below illuminates an interesting relation.
\begin{property} \label{rs:Opt:EVfQv}(Optimal Value vs Q-value) $v^*_{\mathcal{U}^s_p}(s)$ is bounded by the Q-value of $\chi_p(s)$th and $(\chi_p(s)+1)$th actions, that is , 
\[Q^*_{\mathcal{U}^s_p}(s, a_{\chi_p(s)+1}) < v^*_{\mathcal{U}^s_p}(s) \leq Q^*_{\mathcal{U}^s_p}(s, a_{\chi_p(s)}).\]
\end{property}
This special case of the property \ref{rs:EVfQv}, similarly table \ref{tb:Opt:ValVsQ} is special case of table \ref{tb:ValVsQ}.
\begin{table}[H]
\caption{Optimal value function and  Q-value}
  \label{tb:Opt:ValVsQ}
  \centering
  \begin{tabular}{lll}
    \toprule                   
     $v^*(s) = \max_{a}Q^*(s,a) $      & Best value& \\
    $ v^*_{\mathcal{U}^{sa}_p}(s) = \max_a[\alpha_{s,a}-\gamma\beta_{s,a}\kappa_q(v^*_{\mathcal{U}^{sa}_p})-Q^*_{\mathcal{U}^{sa}_p}(s,a) ]$ & Best regularized value&\\
    $ Q^*_{\mathcal{U}^s_p}(s, a_{\chi_p(s)+1}) < v^*_{\mathcal{U}^s_p}(s) \leq Q^*_{\mathcal{U}^s_p}(s, a_{\chi_p(s)})$ & Sandwich!&\\
    \bottomrule
 \end{tabular}
  \begin{tabular}{l}
  where $v^*, Q^*$ is the optimal value function and Q-value respectively \\
  of non-robust MDP.
  \end{tabular}
\end{table}
The same is true for the non-optimal Q-value and value function.

\begin{theorem}\label{rs:srect:greedyPolicy} (Greedy policy) The greedy policy  $\pi^v_{\mathcal{U}^s_p}$ is a threshold policy, that is proportional to the advantage function, that is 
    \[ \pi^v_{\mathcal{U}^s_p}(a|s) \propto \bigm( Q^v(s,a)-(\mathcal{T}^*_{\mathcal{U}^s_p}v)(s)\bigm )^{p-1}\mathbf{1} \bigm(  Q^v(s,a)\geq (\mathcal{T}^*_{\mathcal{U}^s_p}v)(s)\bigm ).\]
   
\end{theorem}
The above theorem is proved in the appendix, and Theorem \ref{rs:srect:optimalPolicy} is its special case. So is table \ref{tb:Opt:pi} special case of table \ref{tb:pi}.
\begin{table}[H]
\caption{Greedy policy at value function $v$}
  \label{tb:pi}
  \centering
  \begin{tabular}{lll}
    \toprule                   
    $\mathcal{U}$     & $\pi^v_\mathcal{U}(a|s) \propto$      & remark \\
    \midrule
    $\mathcal{U}^s_p$ &$ (Q^v(s,a)- (\mathcal{T}^*_{\mathcal{U}}v)(s))^{p-1}\mathbf{1}(A^v_\mathcal{U}(s,a)\geq0)$ &  top actions proportional to\\ & &  $(p-1)$th power of its advantage\\
    &\\
    $\mathcal{U}^s_1$    &$\frac{ \mathbf{1}(A^v_\mathcal{U}(s,a)\geq0)}{\sum_{a}\mathbf{1}(A^v_\mathcal{U}(s,a)\geq0)}$&top  actions with uniform probability  \\
    &\\
    $\mathcal{U}^s_2$     &$\frac{A^v_\mathcal{U}(s,a)\mathbf{1}A^v_\mathcal{U}(s,a)\geq0)}{\sum_{a}A^v_\mathcal{U}(s,a)\mathbf{1}(A^v_\mathcal{U}(s,a)\geq0)}$ & top   actions proportion to advantage \\ 
    &\\
    $\mathcal{U}^s_\infty$     &$  \text{arg}\max_{a\in\mathcal{A}}Q^v(s,a)$&  best action \\ &\\
    $\mathcal{U}^{\mathtt{sa}}_p$    &$  \text{arg}\max_{a}[- \alpha_{sa} -\gamma\beta_{sa} \kappa_q(v) +Q^v(s,a)]$&  best action \\
    \bottomrule
 \end{tabular}
  \begin{tabular}{l}
  where $A^v_\mathcal{U}(s,a) = Q^v(s,a) -(\mathcal{T}^*_\mathcal{U}v)(s)$ and $Q^v(s,a) = R_0(s,a) + \gamma\sum_{s'}P_0(s'|s,a)v(s')$.
  \end{tabular}
\end{table}
The above result states that the greedy policy takes actions that have a positive advantage, so we have.
\begin{align}
    \chi_p(v,s) :=\bigm\lvert \{a \mid \pi^v_{\mathcal{U}^s_p}(a|s) \geq 0\}\bigm\rvert =\bigm\lvert \{a \mid Q^v(s,a) \geq (\mathcal{T}^*_{\mathcal{U}^s_p})v(s)\}\bigm\rvert.
\end{align}
\begin{property} \label{rs:EVfQv}(Greedy Value vs Q-value) $(\mathcal{T}^*_{\mathcal{U}^s_p}v)(s)$ is bounded by the Q-value of $\chi_p(v,s)$th and $(\chi_p(v,s)+1)$th actions, that is , 
\[Q^v(s, a_{\chi_p(v,s)+1}) < (\mathcal{T}^*_{\mathcal{U}^s_p}v)(s) \leq Q^v(s, a_{\chi_p(v,s)}).\]
\end{property}
\begin{table}[H]
\caption{ Greedy value function and  Q-value}
  \label{tb:ValVsQ}
  \centering
  \begin{tabular}{ll}
    \toprule                   
     $(\mathcal{T}^*v)(s) = \max_{a}Q^v(s,a) $      & Best value \\
    $ (\mathcal{T}^*_{\mathcal{U}^{sa}_p})v(s) = \max_a[\alpha_{s,a}-\gamma\beta_{s,a}\kappa_q(v)-Q^v(s,a) ]$ & Best regularized value\\
    $ Q^v(s, a_{\chi_p(v,s)+1}) < (\mathcal{T}^*_{\mathcal{U}^s_p})v(s) \leq Q^v(s, a_{\chi_p(v,s)})$ & Sandwich!\\
    \bottomrule
 \end{tabular}
  \begin{tabular}{l}
  where   $Q^v(s,a_1)\geq,\cdots,\geq Q^v(s,a_A).$
  \end{tabular}
\end{table}
The property below states that we can compute the number of active actions $\chi_p(v,s)$ (and $\chi_p(s)$) directly without computing greedy (optimal) policy.
\begin{property} $\chi_p(v,s)$ is number of actions that has positive advantage, that is 
 \[\chi_p(v,s) := \max\{k\mid \sum_{i=1}^k\bigm(Q^v(s,a_i) - Q^v(s,a_k)\bigm)^p  \leq \sigma^p \},\]
 where $\sigma = \alpha_s + \gamma\beta_s\kappa_q(v)$, and $Q^v(s,a_1)\geq Q^v(s,a_2),\geq \cdots \geq Q(s,a_A).$ 
\end{property}
When uncertainty radiuses ($\alpha_s,\beta_s$) are zero (essentially $\sigma =0$ ), then $\chi_p(v,s) = 1, \forall v,s$, that means, greedy policy taking the best action. In other words, all the robust results reduce to non-robust results as discussed in section \ref{sec:MDPs} as the uncertainty radius becomes zero.

\begin{algorithm}[H]
\caption{Algorithm to compute \texttt{s}-rectangular $L_p$ robust optimal Bellman Operator}\label{alg:fp}
\begin{algorithmic} [1]
 \STATE \textbf{Input:} $\sigma = \alpha_s +\gamma\beta_s\kappa_q(v), \qquad Q(s,a) = R_0(s,a) + \gamma\sum_{s'} P_0(s'|s,a)v(s')$.
 \STATE \textbf{Output} $(\mathcal{T}^*_{\mathcal{U}^s_p}v)(s), \chi_p(v,s)$
\STATE Sort $Q(s,\cdot)$ and label actions such that $Q(s,a_1)\geq Q(s,a_2), \cdots$.
\STATE Set initial value guess $\lambda_1 = Q(s,a_1)-\sigma$ and counter $k=1$.
\WHILE{$k \leq A-1  $ and $\lambda_k \leq Q(s,a_k)$}
    \STATE Increment counter: $k = k+1$
    \STATE Take $\lambda_k$ to be a solution of the following  
    \begin{equation}\label{eq:alg:fp:lk}
    \sum_{i=1}^{k}\bigm(Q(s,a_i) - x\bigm)^p  = \sigma^p, \quad\text{and}\quad x\leq Q(s,a_k).    
    \end{equation}
\ENDWHILE
\STATE Return: $\lambda_k, k $
\end{algorithmic}
\end{algorithm}
\section{Revisiting kernel noise assumption} \label{app:ForbiddenTransition}
\subsection*{Sa-Rectangular Uncertainty}
Suppose at state $s$, we know that it is impossible to have transition (next) to some states (forbidden states $F_{s,a}$) under some action. That is, we have the transition uncertainty set $\mathcal{P}$ and nominal kernel $P_0$  such that 
\begin{align}
   P_0(s'|s,a) =  P(s'|s,a) = 0, \quad \forall P\in\mathcal{P},\forall s' \in F_{s,a}.
\end{align}
Then we define, the kernel noise as 
\begin{align}
    \mathcal{P}_{s,a} = \{ P \mid \lVert P\rVert_p = \beta_{s,a}, \quad \sum_{s'}P(s')=0, \quad P(s")=0, \forall s"\in F_{s,a} \}.
\end{align}
In this case, our $p$-variance function is redefined as 
\begin{align}
    \kappa_p(v,s,a) =& \min_{\lVert P\rVert_p = \beta_{s,a}, \quad  \sum_{s'}P(s')=0, \quad P(s")=0, \quad \forall s"\in F_{s,a}}\langle P,v\rangle\\
    =& \min_{\omega\in\mathbb{R}} \lVert u-\omega\mathbf{1}\rVert_p,\qquad \text{where $u(s) = v(s)\mathbf{1}(s\notin F_{s,a})$.} \\
    =& \kappa_p(u)
\end{align}
This basically says, we consider value of only those states that is allowed (not forbidden) in calculation of $p$-variance. For example, we have
\begin{align}
    \kappa_\infty(v,s,a) & = \frac{\max_{s\notin F_{s,a}}v(s)-\min_{s\notin F_{s,a}}v(s)}{2}.\\
    \end{align}

So theorem 1 of the main paper can be re-stated as 
 \begin{theorem}(Restated) $(\mathtt{Sa})$-rectangular $L_p$ robust Bellman operator is equivalent to reward regularized (non-robust) Bellman operator. That is, using $\kappa_p$ above, we have 
\begin{equation*}\begin{aligned}
    (\mathcal{T}^\pi_{\mathcal{U}^{\mathtt{sa}}_p} v)(s)  =& \sum_{a}\pi(a|s)[  -\alpha_{s,a} -\gamma\beta_{s,a}\kappa_q(v,s,a)  +R_0(s,a) +\gamma \sum_{s'}P_0(s'|s,a)v(s')],\\
    (\mathcal{T}^*_{\mathcal{U}^{\mathtt{sa}}_p} v)(s)  =& \max_{a\in\mathcal{A}}[  -\alpha_{s,a} -\gamma\beta_{s,a}\kappa_q(v,s,a)  +R_0(s,a) +\gamma \sum_{s'}P_0(s'|s,a)v(s')].
\end{aligned}\end{equation*}
\end{theorem}
\subsubsection*{S-Rectangular Uncertainty}
This notion can also be applied to \texttt{s}-rectanular uncertainty, but with little caution. Here, we define forbidden states in state $s$ to be $F_s$  (state dependent) instead of state-action dependent in \texttt{sa}-rectangular case. Here, we define $p$-variance as 
\begin{align}
    \kappa_p(v,s) = \kappa_p(u), \qquad \text{where $u(s) = v(s)\mathbf{1}(s\notin F_s)$.  }
\end{align}
So the theorem 2 can be restated as 
\begin{theorem}(restated) (Policy Evaluation) \texttt{S}-rectangular $L_p$ robust Bellman operator is equivalent to reward regularized (non-robust) Bellman operator, that is 
\begin{equation*}
    (\mathcal{T}^\pi_{\mathcal{U}^s_p} v)(s)  =   -\Bigm(\alpha_s +\gamma\beta_{s}\kappa_q(v,s)\Bigm)\lVert\pi(\cdot|s)\rVert_q  +\sum_{a}\pi(a|s)\Bigm(R_0(s,a) +\gamma \sum_{s'}P_0(s'|s,a)v(s')\Bigm)
\end{equation*}
where $\kappa_p$ is defined above and $\lVert \pi(\cdot|s)\rVert _q$ is $q$-norm of the vector $\pi(\cdot|s)\in\Delta_{\mathcal{A}}$.
\end{theorem}
All the other results (including theorem 4), we just need to replace the old $p$-variance function with new $p$-variance function appropriately.

\section{Application to UCRL}\label{app:UCRL}
In robust MDPs, we consider the minimization over uncertainty set to avoid risk. When we want to discover the underlying kernel by exploration, then we seek optimistic policy, then we consider the maximization over uncertainty set \cite{UCRL,UCRL2,CUCRL}. We refer the reader to the step 3 of the UCRL algorithm \cite{UCRL}, which seeks to find
\begin{align}
    \text{arg}\max_{\pi}\max_{R,P \in\mathcal{U}}\langle \mu ,v^{\pi}_{P,R}\rangle,
\end{align}
where \[\mathcal{U} =\{(R,P)\mid \lvert R(s,a)-R_0(s,a)\rvert \leq \alpha_{s,a},\lvert P(s'|s,a)-P_0(s'|s,a)\rvert \leq \beta_{s,a,s'}, P\in(\Delta_\mathcal{S})^{\mathcal{S}\times\mathcal{A}} \}\] for current estimated  kernel $P_0$ and reward function $R_0$. We refer section 3.1.1 and step 4 of the UCRL 2 algorithm of \cite{UCRL2}, which seeks to find 
\begin{align}
    \text{arg}\max_{\pi}\max_{R,P \in\mathcal{U}}\langle \mu ,v^{\pi}_{P,R}\rangle,
\end{align}
where 
\begin{align*}
    \mathcal{U} =&\{(R,P)\mid \lvert R(s,a)-R_0(s,a)\rvert \leq \alpha_{s,a},\\ &\qquad \lVert P(\cdot|s,a)-P_0(\cdot|s,a)\rVert_1 \leq \beta_{s,a}, P\in(\Delta_\mathcal{S})^{\mathcal{S}\times\mathcal{A}} \}
\end{align*}
The uncertainty radius $\alpha,\beta$ depends on the number of samples of different transitions and observations of the reward. The paper \cite{UCRL} doesn't explain any method to solve the above problem. UCRL 2 algorithm \cite{UCRL2}, suggests to solve it by linear programming that can be very slow. We show that it can be solved by our methods.

The above problem can be tackled as 
following
\begin{align}
    \max_{\pi}\max_{R,P\in\mathcal{U}^{sa}_p}\langle \mu ,v^{\pi}_{P,R}\rangle.
\end{align}
We can define, optimistic Bellman operators as  
\begin{align}
    \Hat{\mathcal{T}}^\pi_{\mathcal{U}}v := \max_{R,P\in\mathcal{U}}v^{\pi}_{P,R},\qquad 
    \Hat{\mathcal{T}}^*_{\mathcal{U}}v :=\max_{\pi} \max_{R,P\in\mathcal{U}}v^{\pi}_{P,R}.
\end{align}
The well definition and contraction of the above optimistic operators may follow directly from their pessimistic (robust) counterparts. We can evaluate above optimistic operators as
\begin{align}
    &(\Hat{\mathcal{T}}^\pi_{\mathcal{U}^{sa}_p}v)(s) = \sum_{a}\pi(a|s)\bigm[R_0(s,a) + \alpha_{s,a}+ \beta_{s,a}\gamma\kappa_q(v) + \sum_{s'}P_0(s'|s,a)v(s')\bigm],\\
    &(\Hat{\mathcal{T}}^*_{\mathcal{U}^{sa}_p}v)(s) = \max_{a}\bigm[R_0(s,a) + \alpha_{s,a}+ \beta_{s,a}\gamma\kappa_q(v) + \sum_{s'}P_0(s'|s,a)v(s')\bigm].
\end{align}
The uncertainty radiuses $\alpha,\beta$ and nominal values $P_0,R_0$ can be found by similar analysis by \cite{UCRL,UCRL2}. We can get the Q-learning from the above results as 
\begin{align}
    Q(s,a) \to R_0(s,a) - \alpha_{s,a} -\gamma\beta_{s,a}\kappa_q(v) +\gamma\sum_{s'}P_0(s'|s,a)\max_{a'}Q(s',a'), 
\end{align}
where $v(s) = \max_{a}Q(s,a)$. From law of large numbers, we know that uncertainty radiuses $\alpha_{s,a},\beta_{s,a}$ behaves as $O(\frac{1}{\sqrt{n}})$ asymptotically with number of iteration $n$. This resembles very closely to UCB VI algorithm \cite{UCRLVI}.
We emphasize that similar optimistic operators can be defined and evaluated for s-rectangular uncertainty sets too.

\section{Q-Learning for sa-rectangular MDPs}\label{app:SALpQL}
In view of Theorem \ref{rs:saLprvi}, we can define $Q^\pi_{\mathcal{U}^{\mathtt{sa}}_p}$, the robust Q-values under policy $\pi$ for $(\mathtt{sa})$-rectangular $L_p$ constrained uncertainty set $\mathcal{U}^{\mathtt{sa}}_p$ as
\begin{equation}\begin{aligned}
    &Q^\pi_{\mathcal{U}^{\mathtt{sa}}_p}(s,a) := -\alpha_{s,a} -\gamma\beta_{s,a}\kappa_q(v^\pi_{\mathcal{U}^{\mathtt{sa}}_p})  +R_0(s,a) +\gamma \sum_{s'}P_0(s'|s,a)v^\pi_{\mathcal{U}^{\mathtt{sa}}_p}(s').
\end{aligned}\end{equation}
This implies that we have the following relation between robust Q-values and robust value function, same as its non-robust counterparts,
\begin{equation}
    v^\pi_{\mathcal{U}^{\mathtt{sa}}_p}(s) = \sum_{a}\pi(a|s)Q^\pi_{\mathcal{U}^{\mathtt{sa}}_p}(s,a).
\end{equation}
Let $Q^*_{\mathcal{U}^{\mathtt{sa}}_p}$ denote the optimal robust Q-values associated with optimal robust value $v^*_{\mathcal{U}^{\mathtt{sa}}_p}$, given as
\begin{equation}\begin{aligned}\label{eq:saLpQ}
    &Q^*_{\mathcal{U}^{\mathtt{sa}}_p}(s,a) := -\alpha_{s,a} -\gamma\beta_{s,a}\kappa_q(v^*_{\mathcal{U}^{\mathtt{sa}}_p})  +R_0(s,a) +\gamma \sum_{s'}P_0(s'|s,a)v^*_{\mathcal{U}^{\mathtt{sa}}_p}(s').
\end{aligned}\end{equation}
It is evident from Theorem \ref{rs:saLprvi} that optimal robust value and optimal robust Q-values satisfies the following relation, same as its non-robust counterparts,
\begin{equation}\begin{aligned}\label{eq:saLpv}
     v^*_{\mathcal{U}^{\mathtt{sa}}_p}(s') = \max_{a\in\mathcal{A}}Q^*_{\mathcal{U}^{\mathtt{sa}}_p}(s,a).
 \end{aligned}\end{equation}
Combining \ref{eq:saLpv} and \ref{eq:saLpQ}, we have optimal robust Q-value recursion as follows
\begin{equation}\begin{aligned}
  &Q^*_{\mathcal{U}^{\mathtt{sa}}_p}(s,a) = -\alpha_{s,a} -\gamma\beta_{s,a}\kappa_q(v^*_{\mathcal{U}^{\mathtt{sa}}_p})  +R_0(s,a) +\gamma \sum_{s'}P_0(s'|s,a)\max_{a\in\mathcal{A}}Q^*_{\mathcal{U}^{\mathtt{sa}}_p}(s,a).
\end{aligned}\end{equation}
The above robust Q-value recursion enjoys similar properties as its non-robust counterparts. 
\begin{corollary}($(\mathtt{sa})$-rectangular $L_p$ regularized Q-learning) Let
\begin{equation*}\begin{aligned}
    Q_{n+1}(s,a)  =  R_0(s,a)- \alpha_{sa} -\gamma\beta_{sa}\kappa_q(v_n) + \gamma\sum_{s'}P_0(s'|s,a)\max_{a\in\mathcal{A}}Q_n(s',a),
\end{aligned}\end{equation*}
where  $ v_{n}(s) = \max_{a\in\mathcal{A}}Q_{n}(s,a) $, then $Q_n$ converges to $Q^*_{\mathcal{U}^{\mathtt{sa}}_p}$ linearly. 
\end{corollary}
Observe that the above Q-learning equation is exactly the same as non-robust MDP except the reward penalty. Recall that $\kappa_1(v) = 0.5(\max_{s}v(s) - \min_{s}v(s))$ is difference between peak to peak values and $\kappa_2(v)$ is variance of $v$, that can be easily estimated. Hence, model free algorithms for $(\mathtt{sa})$-rectangular $L_p$ robust MDPs for $p=1,2$, can be derived easily from the above results.
This implies that $(\mathtt{sa})$-rectangular $L_1$ and $L_2$ robust MDPs are as easy as non-robust MDPs.

\section{Model Based Algorithms}\label{app:ModelBasedAlgorithms}
In this section, we assume that we know the nominal transitional kernel and nominal reward function. Algorithm \ref{alg:SALp}, algorithm \ref{alg:SLp} is model based algorithm for $(\mathtt{sa})$-rectangular  and $\mathtt{s}$ rectangular $L_p$ robust MDPs respectively. It is explained in the algorithms, how to get deal with specail cases $(p=1,2,\infty)$ in a easy way.

\begin{algorithm}[H]
\caption{Model Based Q-Learning Algorithm for SA Rectangular $L_p$ Robust MDP}\label{alg:SALp}
\begin{algorithmic} [1]
\STATE \textbf{Input}: $\alpha_{s,a},\beta_{s,a}$ are uncertainty radius in reward and transition kernel respectively in state $\mathtt{s}$ and action $a$. Transition kernel $P$ and reward vector $R$. Take initial $Q$-values $Q_0$ randomly and $v_0(s) = \max_{a}Q_0(s,a).$\\
\WHILE{ not converged}
    \STATE Do binary search in $[\min_{s}v_n(s), \max_{s}v_n(s)]$ to get $q$-mean $\omega_n$, such that
    \begin{equation}\label{alg:SLP:eq:kappa}
        \sum_{s}\frac{(v_n(s)-\omega_n)}{|v_n(s)-\omega_n|}|v_n(s) - \omega_n|^{\frac{1}{p-1}} = 0.
    \end{equation}
    \STATE Compute $q$-variance: $\qquad \kappa_n = \lVert v-\omega_n\rVert_q$.
    \STATE Note: For $p=1,2,\infty$, we can compute $\kappa_n$ exactly in closed from, see table \ref{tb:kappa}.
\FOR{$s \in \mathcal{S}$ }
    \FOR{$a \in \mathcal{A}$ }
        \STATE Update Q-value as \[Q_{n+1}(s,a) = R_0(s,a)  - \alpha_{sa} -\gamma\beta_{sa}\kappa_n + \gamma \sum_{s'}P_0(s'|s,a)\max_{a}Q_n(s',a).\]
    \ENDFOR
    \STATE Update value as \[v_{n+1}(s) = \max_{a}Q_{n+1}(s,a).\]
\ENDFOR
    \[n \to n+1\]
\ENDWHILE
\end{algorithmic}
\end{algorithm}

\begin{algorithm}[H]
\caption{Model Based Algorithm for S Rectangular $L_p$ Robust MDP}\label{alg:SLp}
\begin{algorithmic} [1]
\STATE Take initial $Q$-values $Q_0$ and value function $v_0$ randomly. 
\STATE \textbf{Input}: $\alpha_{s},\beta_{s}$ are uncertainty radius in reward and transition kernel respectively in state $\mathtt{s}$.\\
\WHILE{ not converged}
\STATE Do binary search in $[\min_{s}v_n(s), \max_{s}v_n(s)]$ to get $q$-mean $\omega_n$, such that
    \begin{equation}\label{alg:SLP:eq:kappa}
        \sum_{s}\frac{(v_n(s)-\omega_n)}{|v_n(s)-\omega_n|}|v_n(s) - \omega_n|^{\frac{1}{p-1}} = 0.
    \end{equation}
\STATE Compute $q$-variance: $\qquad \kappa_n = \lVert v-\omega_n\rVert_q$.
\STATE Note: For $p=1,2,\infty$, we can compute $\kappa_n$ exactly in closed from, see table \ref{tb:kappa}.
\FOR{$s \in \mathcal{S}$ }
    \FOR{$a \in \mathcal{A}$ }
        \STATE Update Q-value as \begin{equation}\label{alg:SLP:eq:Qupdate}
        Q_{n+1}(s,a) =R_0(s,a) +\gamma\sum_{s'}P_0(s'|s,a) v_{n+1}(s').
        \end{equation}
    \ENDFOR
    \STATE Sort actions in decreasing order of the Q-value, that is
    \begin{equation}\label{alg:SLP:eq:Qsort}
    Q_{n+1}(s,a_i)\geq Q_{n+1}(s,a_{i+1}).
    \end{equation}
    \STATE Value evaluation:
    \begin{equation}\label{alg:SLP:eq:valEval}
    v_{n+1}(s) = x \quad \text{such that }\quad (\alpha_s +\gamma\beta_{s}\kappa_n)^{p} = \sum_{Q_{n+1}(s,a_i)\geq x}|Q_{n+1}(s,a_i) - x|^{p}.
    \end{equation}
    \STATE Note: We can compute $v_{n+1}(s)$ exactly in closed from for $p=\infty$ and for $p=1,2$, we can do the same using algorithm \ref{alg:f1},\ref{alg:f2} respectively, see table \ref{tb:val}.
\ENDFOR
\[n \to n+1\]
\ENDWHILE
\end{algorithmic}
\end{algorithm}

\begin{algorithm}[H]
\caption{Model based algorithm for \texttt{s}-recantangular $L_1$ robust MDPs}\label{alg:Mb:SL1}
\begin{algorithmic} [1]
\STATE Take initial value function $v_0$ randomly and start the counter $n=0$. 
\WHILE{ not converged}
\STATE Calculate $q$-variance: $ \qquad \kappa_n = \frac{1}{2}\bigm[\max_{s}v_n(s) - \min_{s}v_n(s)\bigm]$
\FOR{$s \in \mathcal{S}$ }
    \FOR{$a \in \mathcal{A}$ }
        \STATE Update Q-value as \begin{equation}\label{alg:SLP:eq:Qupdate}
        Q_{n}(s,a) =R_0(s,a) +\gamma\sum_{s'}P_0(s'|s,a) v_{n}(s').
        \end{equation}
    \ENDFOR
    \STATE Sort actions in state $s$, in decreasing order of the Q-value, that is
    \begin{equation}\label{alg:SLP:eq:Qsort}
    Q_{n}(s,a_1)\geq Q_{n}(s,a_{2}),\cdots \geq Q_{n}(s,a_A).
    \end{equation}
    \STATE Value evaluation:
    \begin{equation}\label{alg:SLP:eq:valEval}
    v_{n+1}(s) = \max_{m}\frac{\sum_{i=1}^m Q_{n}(s,a_i) - \alpha_s -\beta_s\gamma\kappa_n}{m}.
    \end{equation}
    \STATE Value evaluation can also be done using algorithm \ref{alg:f1}.
\ENDFOR
\[n \to n+1\]
\ENDWHILE
\end{algorithmic}
\end{algorithm}

\section{Experiments}\label{app:experiments}
The table 4 contains relative cost (time) of robust value iteration w.r.t. non-robust MDP, for randomly generated kernel and reward function with the number of states $S$ and the number of action $A$.

\begin{table}\label{tb:rlt}
  \caption{Relative running cost (time) for value iteration}
  \centering
  \begin{tabular}{llllll}
    \toprule                   
    $\mathcal{U}$     & S=10 A=10 & S=30 A=10 & S=50 A=10 & S=100 A=20     & remark \\
    \midrule
    non-robust & 1 &1 &1 &1 &      \\&\\
    $\mathcal{U}^{sa}_{\infty}$ by LP & 1374 &2282 &2848 &6930 &   lp    \\&\\
    $\mathcal{U}^{sa}_{1}$ by LP & 1438 &6616 &6622 &16714 &      lp\\&\\
    $\mathcal{U}^{s}_{1}$ by LP& 72625 &629890 &4904004 &NA &  lp/minimize     \\&\\
    $\mathcal{U}^{sa}_{1}$& 1.77 &1.38 &1.54 &1.45 & closed form     \\&\\
    $\mathcal{U}^{sa}_{2}$& 1.51 &1.43 &1.91 &1.59 & closed form     \\&\\
    $\mathcal{U}^{sa}_{\infty}$& 1.58 &1.48 &1.37 &1.58 &    closed form  \\&\\
    $\mathcal{U}^{s}_{1}$& 1.41 &1.58 &1.20 &1.16 &closed form      \\&\\
    $\mathcal{U}^{s}_{2}$& 2.63 &2.82 &2.49 &2.18 &     closed form \\&\\
    $\mathcal{U}^{s}_{\infty}$& 1.41 &3.04 &2.25 &1.50 &    closed form  \\&\\
    $\mathcal{U}^{sa}_{5}$& 5.4 &4.91 &4.14 &4.06 &  binary search     \\&\\
    $\mathcal{U}^{sa}_{10}$& 5.56 &5.29 &4.15 &3.26 &binary search      \\&\\
    $\mathcal{U}^{s}_{5}$& 33.30 &89.23 &40.22 &41.22 & binary search     \\&\\
    $\mathcal{U}^{s}_{10}$& 33.59 &78.17 &41.07 &41.10 &    binary search  \\&\\
    \bottomrule
  \end{tabular}
  \begin{tabular}{l}
       lp stands for scipy.optimize.linearprog
  \end{tabular}
  \end{table}

\subsection*{Notations} 
S : number of state, 
A: number of actions, 
$\mathcal{U}^{sa}_p$ LP:   Sa rectangular $L_p$ robust MPDs by Linear Programming, 
$\mathcal{U}^{s}_p$ LP:  S rectangular $L_p$ robust MPDs by Linear Programming and other numerical methods, 
$\mathcal{U}^{sa/s}_{p=1,2,\infty}$ : Sa/S rectangular $L_1/L_2/L_\infty$ robust MDPs by closed form method (see table 2, theorem 3)
$\mathcal{U}^{sa/s}_{p=5,10}$ : Sa/S rectangular $L_5/L_{10}$ robust MDPs by binary search (see table 2, theorem 3 of the paper)

\subsection*{Observations}
1. Our method for s/sa rectangular $L_1/L_2/L_\infty$ robust MDPs takes almost same (1-3 times) the time as non-robust MDP for one iteration of value iteration. This confirms our complexity analysis (see table 4 of the paper)
2.  Our binary search method for sa rectangular $L_5/L_{10}$ robust MDPs takes around $4-6$ times more time than non-robust counterpart. This is due to extra iterations required to find p-variance function $\kappa_p(v)$ through binary search.
3.  Our binary search method for s rectangular $L_5/L_{10}$ robust MDPs takes around $30-100$ times more time than non-robust counterpart. This is due to extra iterations required to find p-variance function $\kappa_p(v)$ through binary search and Bellman operator.
4.  One common feature of our method is that time complexity scales moderately as guranteed through our complexity analysis.
5.  Linear programming methods for sa-rectangualr $L_1/L_\infty$ robsust MDPs take atleast 1000 times more than our methods for small state-action space, and it scales up very fast.
6. Numerical methods (Linear programming for minimization over uncertianty and 'scipy.optimize.minimize' for maximization over policy) for s-rectangular $L_1$ robust MDPs take 4-5 order more time than our mehtods (and non-robust MDPs) for very small state-action space, and scales up too fast. The reason is obvious, as it has to solve two optimization, one minimization over uncertainty and other maximization over policy, whereas in the sa-rectangular  case, only minimization over uncertainty is required.  This confirms that s-rectangular uncertainty set is much more challenging.

\subsection*{Rate of convergence}
The rate of convergence for all were approximately the same as $0.9 = \gamma$, as predicted by theory. And it is well illustrated by the relative rate of convergence w.r.t. non-robust by the table \ref{tb:roc}.

\begin{table}\label{tb:roc}
  \caption{Relative running cost (time) for value iteration}
  \centering
  \begin{tabular}{llll}
    \toprule                   
    $\mathcal{U}$     & S=10 A=10 &  S=100 A=20     & remark \\
    \midrule
    non-robust & 1 &1 &      \\&\\
    $\mathcal{U}^{sa}_{1}$& 0.999 &0.999 & closed form     \\&\\
    $\mathcal{U}^{sa}_{2}$&0.999 &0.999 & closed form     \\&\\
    $\mathcal{U}^{sa}_{\infty}$& 1.000 &0.998  &    closed form  \\&\\
    $\mathcal{U}^{s}_{1}$& 0.999 &0.999 &
    closed form      \\&\\
    $\mathcal{U}^{s}_{2}$& 0.999 &0.999 &    closed form \\&\\
    $\mathcal{U}^{s}_{\infty}$& 1.000 &0.998 &    closed form  \\&\\
    $\mathcal{U}^{sa}_{5}$& 0.999 &0.995 & binary search     \\&\\
    $\mathcal{U}^{sa}_{10}$& 1.000 &0.999  &binary search      \\&\\
    $\mathcal{U}^{s}_{5}$& 1.000 &0.999  & binary search     \\&\\
    $\mathcal{U}^{s}_{10}$& 1.000 &0.995 &   binary search  \\&\\
    \bottomrule
  \end{tabular}
  \end{table}
 
In the above experiments, Bellman updates for sa/s rectangular $L_1/L_2/L_\infty$ were done in closed form, and for $L_5/L_{10}$ were done by binary search as suggested by table 2 and theorem 3. 

Note: Above experiments' results are for few runs, hence containing some stochasticity but the general trend is clear. In the final version, we will do averaging of many runs to minimize the stochastic nature. Results for many different runs can be found at https://github.com/******.

Note that the above experiments were done without using too much parallelization. There is ample scope to fine-tune and improve the performance of robust MDPs. The above experiments confirm the theoretical complexity provided in Table 4 of the paper. The codes and results can be found at https://github.com/******.

\subsection*{Experiments parameters}  Number of states $S$ (variable), number of actions $A$ (variable), transition kernel and reward function generated randomly,  discount factor $0.9$,  uncertainty radiuses =$0.1$ (for all states and action, just for convenience ),  number of iterations = 100,  tolerance for binary search =  $0.00001$
\subsection*{Hardware} The experiments are done on the following hardware: Intel(R) Core(TM) i5-4300U CPU @ 1.90GHz 64 bits, memory 7862MiB
Software: Experiments were done in python, using numpy, scipy.optimize.linprog for Linear programmig for policy evalution in s/sa rectangular robust MDPs, scipy.optimize.minize and scipy.optimize.LinearConstraints for policy improvement in s-rectangular $L_1$ robust MDPs.

\section{Extension to Model Free Settings}
Extension of Q-learning (in section \ref{app:SALpQL} ) for \texttt{sa}-rectangular MDPs to model free setting can easily done similar to \cite{Rcontamination}, also policy gradient method can be obtained as \cite{PG_RContamination}. The only thing, we need to do, is to be able to compute/estimate $\kappa_q$ online. It can be estimated using an ensemble (samples). Further,   $\kappa_2$ can be estimated by the estimated mean and the estimated second moment. $\kappa_\infty$ can be estimated by tracking maximum and minimum values.

For \texttt{s}-rectangular case too, we can obtain model-free algorithms easily, by estimating $\kappa_q$ online and keeping track of Q-values and value function. The convergence analysis may be similar to \cite{Rcontamination}, especially for \texttt{sa}-rectangular case, and for the other, it would be two time scale, which can be dealt with techniques in \cite{borkarBook}. We leave this for future work. It would be interesting to obtain policy gradient methods for this case, which we believe can be obtained from the policy evaluation theorem.

\section{p-variance}
\label{app:pvarianceSection}
Recall that $\kappa_p$ is defined as follows 
\[\kappa_p(v) =\min_{w} \lVert v-\omega\mathbf{1}\rVert_p = \lVert v-\omega_p\rVert_p.\]
Now, observe that
\begin{equation}\begin{aligned}
&\frac{\partial \lVert v -\omega\rVert_p}{\partial \omega}  = 0\\
\implies &  \sum_{s}sign(v(s)-\omega)|v(s) - \omega|^{p-1} = 0,\\
\implies &  \sum_{s}sign(v(s)-\omega_p(v))|v(s) - \omega_q(v)|^{p-1} = 0.
\end{aligned}
\end{equation}
For $p=\infty$,  we have
\begin{equation}\begin{aligned}
&\lim_{p\to\infty}\Bigm\lvert\sum_{s}sign\bigm(v(s)-\omega_\infty(v)\bigm)\bigm\lvert v(s) - \omega_\infty(v)\bigm\rvert^{p}\Bigm\rvert^\frac{1}{p} = 0 \\
=&\bigm(\max_{s}\lvert v(s) - \omega_\infty(v)\rvert\bigm)\lim_{p\to\infty}\Bigm\lvert\sum_{s}sign\bigm(v(s)-\omega_\infty(v)\bigm)\Bigm(\frac{\lvert v(s) - \omega_\infty(v)\bigm\rvert}{\max_{s}\lvert v(s) - \omega_\infty(v)\rvert}\Bigm)^{p}\Bigm\rvert^\frac{1}{p} \\
&\text{Assuming $\max_{s}\lvert v(s) - \omega_\infty(v)\rvert \neq 0$ otherwise $\omega_\infty = v(s)=v(s'),\quad \forall s,s'$ }\\
\implies&\lim_{p\to\infty}\Bigm\lvert\sum_{s}sign\bigm(v(s)-\omega_\infty(v)\bigm)\Bigm(\frac{\lvert v(s) - \omega_\infty(v)\bigm\rvert}{\max_{s}\lvert v(s) - \omega_\infty(v)\rvert}\Bigm)^{p}\Bigm\rvert^\frac{1}{p} = 0 \\
&\text{To avoid technical complication, we assume $\max_{s}v(s)>v(s)< \min_{s}v(s), \quad \forall s$}\\
\implies&\lim_{p\to\infty} \lvert\max_{s} v(s) - \omega_\infty(v)\rvert =\lim_{p\to\infty}\lvert\min_{s} v(s) - \omega_\infty(v)\rvert\\
\implies& \max_{s} v(s) - \lim_{q\to\infty}\omega_\infty(v) =-(\min_{s} v(s) - \lim_{p\to\infty}\omega_\infty(v)),\qquad \text{(managing signs)}\\
\implies&\lim_{p\to\infty}\omega_\infty(v) = \frac{\max_{s}v(s) +\min_{s}v(s)}{2}.
\end{aligned}\end{equation}

\begin{equation}\begin{aligned}
    \kappa_\infty(v) =& \lVert v-\omega_{\infty}\mathbf{1}\rVert_\infty\\
    =& \lVert v-\frac{\max_{s}v(s) + \min_{s}v(s)}{2}\mathbf{1}\rVert_\infty, \qquad \text{(putting in value of $\omega_\infty$)}\\
    =&\frac{\max_{s}v(s) - \min_{s}v(s)}{2}
\end{aligned}\end{equation}

For $p=2$, we have
\begin{equation}\begin{aligned}
    \kappa_2(v) =& \lVert v-\omega_{2}\mathbf{1}\rVert_2\\
    =& \lVert v-\frac{\sum_{s}v(s)}{S}\mathbf{1}\rVert_2,\\
    =&\sqrt{\sum_{s}(v(s) -\frac{\sum_{s}v(s)}{S})^2}
\end{aligned}\end{equation}

For $p=1$, we have
\begin{equation}\begin{aligned}
&\sum_{s\in\mathcal{S}} sign\bigm(v(s) - \omega_1(v)\bigm) = 0 \\
\end{aligned}\end{equation}
Note that there may be more than one values of $\omega_1(v)$ that satisfies the above equation and each solution does equally good job (as we will see later). So we will pick one ( is median of $v$) according to our convenience as
\[ \omega_1(v) = \frac{v(s_{\lfloor (S+1)/2\rfloor}) +v(s_{\lceil (S+1)/2\rceil})}{2} \quad \text{where} \quad v(s_i)\geq v(s_{i+1}) \quad \forall i.\] 

\begin{equation}\begin{aligned}
    \kappa_1(v) =& \lVert v-\omega_{1}\mathbf{1}\rVert_1\\
    =& \lVert v-med(v)\mathbf{1}\rVert_1, \qquad \text{(putting in value of $\omega_0$, see table \ref{tb:mean})}\\
    =&\sum_{s}\lvert v(s)-med(v)\rvert\\
    =&\sum_{i=1}^{\lfloor (S+1)/2\rfloor} (v(s)-med(v))  +\sum_{\lceil (S+1)/2\rceil}^{S} (med(v) -v(s)) \\
    =&\sum_{i=1}^{\lfloor (S+1)/2\rfloor} v(s)  -\sum_{\lceil (S+1)/2\rceil}^{S} v(s) 
\end{aligned}\end{equation}
where 
$med(v) := \frac{v(s_{\lfloor (S+1)/2\rfloor}) +v(s_{\lceil (S+1)/2\rceil})}{2} \quad \text{where} \quad v(s_i)\geq v(s_{i+1}) \quad \forall i$ is median of $v$. The results are summarized in table \ref{tb:kappa} and \ref{tb:mean}.
\begin{table}
  \caption{$p$-mean, where $v(s_i)\geq v(s_{i+1}) \quad \forall i.$}
  \label{tb:mean}
  \centering
  \begin{tabular}{lll}
    \toprule                   
    $x$     & $\omega_x(v)$     & remark \\
    \midrule
    $p$ & $\sum_{s}sign(v(s)-\omega_p(v))\lvert v(s) - \omega_p(v)\rvert^{\frac{1}{p-1}} = 0 $ & Solve by binary search\\\\
    $1$     & $\frac{v(s_{\lfloor (S+1)/2\rfloor}) +v(s_{\lceil (S+1)/2\rceil})}{2} $      & Median  \\\\
    2     &  $\frac{\sum_{s}v(s)}{S}    $ & Mean\\\\
    $\infty$     & $\frac{\max_{s}v(s) + \min_{s}v(s)}{2}$      &  Average of peaks  \\
    \bottomrule
  \end{tabular}
\end{table}


\subsection{p-variance function  and kernel noise}

\begin{lemma}
\label{regfn}
$q$-variance function $\kappa_q$ is the solution of the following optimization problem (kernel noise),
\[\kappa_q(v) = -\frac{1}{\epsilon}\min_{c}\langle c,v\rangle, \qquad \lVert c\rVert_p\leq \epsilon, \qquad \sum_{s}c(s) = 0.\]
\end{lemma}
\begin{proof}
Writing Lagrangian $L$, as
\[L := \sum_{s}c(s)v(s) + \lambda\sum_{s}c(s) + \mu(\sum_{s}\lvert c(s)\rvert^p -\epsilon^p),\]
where $\lambda \in\mathbb{R}$ is the multiplier for the constraint $\sum_{s}c(s) = 0$ and $\mu \geq 0$ is the multiplier for the inequality constraint $\lVert c\lVert_q \leq \epsilon.$ Taking its derivative, we have
\begin{equation}\begin{aligned}
    \frac{\partial L}{\partial c(s)} = v(s) + \lambda + \mu p \lvert c(s)\rvert^{p-1}\frac{c(s)}{\lvert c(s)\rvert}
\end{aligned}\end{equation}
From the KKT (stationarity)  condition, the solution $c^
*$ has zero derivative, that is 
\begin{equation}\begin{aligned}\label{app:eq:kappa:LagDer}
     v(s) + \lambda + \mu p \lvert c^*(s)\rvert^{p-1}\frac{c^*(s)}{\lvert c^*(s)\rvert} = 0, \qquad \forall s\in\mathcal{S}.
\end{aligned}\end{equation}

Using Lagrangian derivative equation \eqref{app:eq:kappa:LagDer}, we have
\begin{equation}\begin{aligned}
    & v(s) + \lambda + \mu p \lvert c^*(s)\rvert^{p-1}\frac{c^*(s)}{\lvert c^*(s)\rvert} = 0 \\
    \implies & \sum_{s}c^*(s)[v(s)  + \lambda + \mu p \lvert c^*(s)\rvert^{p-1}\frac{c^*(s)}{\lvert c^*(s)\rvert}] = 0, \qquad \text{(multiply with $c^*(s)$ and summing )}\\
     \implies & \sum_{s}c^*(s)v(s)  + \lambda\sum_{s}c^*(s) + \mu p \sum_{s}\lvert c^*(s)\rvert^{p-1}\frac{(c^*(s))^2}{\lvert c^*(s)\rvert} = 0\\
    \implies & \langle c^*,v\rangle + \mu p \sum_{s}\lvert c^*(s)\rvert^{p} = 0 \qquad \text{(using $\sum_{s}c^*(s) =0$ and $(c^*(s))^2 = \lvert c^*(s)\rvert^2$ )}\\
    \implies & \langle c^*,v\rangle = -\mu p \epsilon^p, \qquad \text{(using $\sum_{s}\lvert c^*(s)\rvert^p= \epsilon^p$ ).} 
\end{aligned}\end{equation}
It is easy to see that $\mu \geq 0$, as minimum value of the objective must not be positive ( at $c=0$, the objective value is zero).
Again we use Lagrangian derivative \eqref{app:eq:kappa:LagDer} and try to get the objective value  ($-\mu p \epsilon^p$) in terms of $\lambda$, as 
\begin{equation}\begin{aligned}
    &v(s) + \lambda + \mu p \lvert c^*(s)\rvert^{p-1}\frac{c^*(s)}{\lvert c^*(s)\rvert} = 0\\
    \implies &\lvert c^*(s)\rvert^{p-2}c^*(s) =-\frac{v(s) + \lambda}{\mu p}, \qquad \text{(re-arranging terms)} \\
    \implies &\sum_{s}|(\lvert c^*(s)\rvert^{p-2}c^*(s))|^{\frac{p}{p-1}} =\sum_{s}|-\frac{v(s) + \lambda}{\mu p}|^{\frac{p}{p-1}}, \qquad \text{(doing $\sum_s\lvert\cdot\rvert^\frac{p}{p-1}$ )} \\
    \implies & \lVert c^*\rVert^p_p =\sum_{s}|-\frac{v(s) + \lambda}{\mu p}|^{\frac{p}{p-1}} = \sum_{s}|\frac{v(s) + \lambda}{\mu p}|^{q} =\frac{\lVert v + \lambda\rVert^q_q}{|\mu p|^q} \\
    \implies & |\mu p|^q\lVert c^*\rVert^p_p = \lVert v + \lambda\rVert^q_q , \qquad \text{(re-arranging terms)}\\
    \implies & |\mu p|^q \epsilon^p=\lVert v + \lambda\rVert^q_q, \qquad \text{(using $\sum_{s}\lvert c^*(s)\rvert^p= \epsilon^p$ )}\\
    \implies & \epsilon (\mu p \epsilon^{p/q})=  \epsilon \lVert v + \lambda\rVert_q\qquad \text{(taking $\frac{1}{q}$the power then multiplying with $\epsilon$)} \\
    \implies &\mu p \epsilon^p=\epsilon \lVert v + \lambda\rVert_q.
\end{aligned}\end{equation}

Again, using Lagrangian derivative \eqref{app:eq:kappa:LagDer} to solve for $\lambda$, we have
\begin{equation}\begin{aligned}
    &v(s) + \lambda + \mu p \lvert c^*(s)\rvert^{p-1}\frac{c^*(s)}{\lvert c^*(s)\rvert} = 0\\
    \implies &\lvert c^*(s)\rvert^{p-2}c^*(s) =-\frac{v(s) + \lambda}{\mu p} , \qquad \text{(re-arranging terms)}\\
    \implies & \lvert c^*(s)\rvert = |\frac{v(s) + \lambda}{\mu p}|^{\frac{1}{p-1}},  \quad \text{(looking at absolute value)}\\
    \qquad\qquad &\quad \text{and}\quad \frac{c^*(s)}{\lvert c^*(s)\rvert} = -\frac{v(s)+\lambda}{|v(s)+\lambda|}, \quad \text{(looking at sign: and note $\mu,p\geq 0$)}\\
    \implies &  \sum_{s}\frac{c^*(s)}{\lvert c^*(s)\rvert}\lvert c^*(s)\rvert = -\sum_{s}\frac{v(s)+\lambda}{|v(s)+\lambda|}|\frac{v(s) + \lambda}{\mu p}|^{\frac{1}{p-1}}, \qquad \text{(putting back)}\\
    \implies &  \sum_{s}c^*(s) = -\sum_{s}\frac{v(s)+\lambda}{|v(s)+\lambda|}|\frac{v(s) + \lambda}{\mu p}|^{\frac{1}{p-1}}, \\
    \implies & \sum_{s}\frac{v(s)+\lambda}{|v(s)+\lambda|}|v(s) + \lambda|^{\frac{1}{p-1}} = 0, \qquad \text{( using $\sum_i c^*(s) = 0$)}
\end{aligned}\end{equation}
Combining everything, we have 
\begin{equation}\label{kappa}\begin{aligned}
    &-\frac{1}{\epsilon}\min_{c}\langle c,v\rangle, \qquad \lVert c\rVert_p\leq \epsilon, \qquad \sum_{s}c(s) = 0\\
    = &\lVert v -\lambda\rVert_q, \quad\text{such that}\quad \sum_{s}sign(v(s)-\lambda)|v(s) - \lambda|^{\frac{1}{p-1}} = 0.
\end{aligned}\end{equation}
Now, observe that
\begin{equation}\begin{aligned}
&\frac{\partial \lVert v -\lambda\rVert_q}{\partial \lambda}  = 0\\
\implies &  \sum_{s}sign(v(s)-\lambda)|v(s) - \lambda|^{\frac{1}{p-1}} = 0,\\
\implies & \kappa_q(v) = \lVert v -\lambda\rVert_q, \quad\text{such that}\quad \sum_{s}sign(v(s)-\lambda)|v(s) - \lambda|^{\frac{1}{p-1}} = 0.
\end{aligned}
\end{equation}
The last equality follows from the convexity of p-norm $\lVert \cdot \rVert_q$, where every local minima is global minima. 


For the sanity check, we re-derive things for $p=1$ from scratch.
For $p=1$, we have
\begin{equation}\begin{aligned}
&-\frac{1}{\epsilon}\min_{c}\langle c,v\rangle, \qquad \lVert c\rVert_1\leq \epsilon, \qquad \sum_{s}c(s) = 0.\\
=& - \frac{1}{2}(\min_{s}v(s) - \max_{s}v(s))\\
=& \kappa_1(v).   
\end{aligned}\end{equation}
It is easy to see the above result, just by inspection.
\end{proof}

\subsection{Binary search for p-mean and estimation of p-variance}\label{app:BSkappa}
If the function $f:[-B/2,B/2] \to \mathbb{R}, B\in\mathbb{R}$ is monotonic (WLOG let it be monotonically decreasing) in a bounded domain, and it has a unique root $x^*$ s.t. $f(x^*) = 0$. Then we can find $x$ that is an $\epsilon$-approximation $x^*$ (i.e. $\lVert x - x^*\rVert \leq \epsilon$ ) in $O(B/\epsilon)$ iterations. Why?  Let $x_0=0$ and
\[x_{n+1} := \begin{cases} \frac{-B + x_n}{2} & \text{if}\quad f(x_n) > 0\\
\frac{B + x_n}{2} & \text{if}\quad f(x_n) < 0\\
 x_n & \text{if}\quad f(x_n) = 0\\\end{cases}.\]
It is easy to observe that $\lVert x_n - x^*\rVert \leq B(1/2)^n$. This is proves the above claim. This observation will be referred to many times.\\
Now, we move to the main claims of the section.

\begin{proposition}\label{prp:SLpValMeanEval}
The function   \[h_p(\lambda):=\sum_{s}sign\bigm(v(s)-\lambda\bigm)\bigm\lvert v(s) - \lambda\bigm\rvert^{p}\] is monotonically strictly decreasing and also has a root in the range $[\min_{s}v(s),\max_{s}v(s)]$.
\end{proposition}
\begin{proof}
\begin{equation}\begin{aligned}
    &h_p(\lambda) = \sum_{s}\frac{v(s)-\lambda}{|v(s) - \lambda|}|v(s) - \lambda|^{p}\\
    &\frac{dh_p}{d\lambda}(\lambda) = -p\sum_{s}|v(s) - \lambda|^{p-1} \leq 0,\qquad \forall p\geq 0.
\end{aligned}\end{equation}
Now, observe that $h_p(\max_{s}v(s)) \leq 0$ and $h_p(\min_{s}v(s)) \geq 0$, hence by $h_p$ must have a root in the range $[\min_{s}v(s),\max_{s}v(s)]$ as the function is continuous.
\end{proof}

The above proposition ensures that a root $\omega_p(v)$ can be easily found by binary search between $[\min_{s}v(s),\max_{s}v(s)]$.

Precisely, $\epsilon$ approximation of $\omega_p(v)$ can be found in  $O(\log(\frac{\max_{s}v(s)-\min_{s}v(s)}{\epsilon}))$ number of iterations of binary search. And one evaluation of the function $h_p$ requires $O(S)$ iterations. And we have finite state-action space and bounded reward hence WLOG we can assume $\lvert\max_{s}v(s)\rvert, \lvert\min_{s}v(s)\rvert$ are bounded by a constant.  Hence, the complexity to approximate $\omega_p$ is $O(S\log(\frac{1}{\epsilon}))$.

 Let  $\hat{\omega}_{p}(v)$ be an $\epsilon$-approximation  of $\omega_p(v)$, that is 
\[\bigm\lvert \omega_p(v) - \hat{\omega}_p(v) \bigm\rvert \leq \epsilon.\]

And let $\hat{\kappa}_p(v)$ be approximation of $\kappa_p(v)$ using approximated mean, that is, 
\[\hat{\kappa}_p(v) := \lVert v -\hat{\omega}_{p}(v)\mathbf{1}\rVert_p.\]
Now we will show that $\epsilon$ error in calculation of $p$-mean $\omega_p$, induces $O(\epsilon)$ error in estimation of $p$-variance $\kappa_p$. Precisely,
\begin{equation}\begin{aligned}
    \Bigm\lvert \kappa_p(v)-\hat{\kappa}_p(v)\Bigm\rvert =& \Bigm\lvert \bigm\lVert v -\omega_{p}(v)\mathbf{1}\bigm\rVert_p -\bigm\lVert v -\hat{\omega}_{p}(v)\mathbf{1}\bigm\rVert_p\Bigm\rvert \\
    \leq& \bigm\lVert \omega_{p}(v)\mathbf{1}-\hat{\omega}_{p}(v)\mathbf{1}\bigm\rVert_p,\qquad \text{(reverse triangle inequality)}\\ 
    =&  \bigm\lVert\mathbf{1}\bigm\rVert_p\bigm\lvert \omega_{p}(v)-\hat{\omega}_{p}(v)\bigm\rvert\\
    \leq&  \bigm\lVert\mathbf{1}\bigm\rVert_p\epsilon\\
    = &S^\frac{1}{p}\epsilon \leq S\epsilon. 
\end{aligned}\end{equation}
For general $p$, an $\epsilon$ approximation of $\kappa_p(v)$ can be calculated in 
$O(S\log(\frac{S}{\epsilon})$ iterations. Why? We will estimate mean $\omega_p$ to an $\epsilon/S$ tolerance (with cost $O(S\log(\frac{S}{\epsilon})$ ) and then approximate the $\kappa_p$ with this approximated mean (cost $O(S)$).

\section{ Lp Water Filling/Pouring  lemma}\label{app:waterPouringSection}
In this section, we are going to discuss the following optimization problem,
\[\max_{c}-\alpha \lVert c\rVert_q + \langle c,b\rangle \qquad \text{such that }\qquad \sum_{i=1}^Ac_i = 1,\quad c_i\geq 0,\quad \forall i\]
where $\alpha\geq 0$, referred as $L_p$-water pouring problem. We are going to assume WLOG that $b$ is sorted component wise, that is $b_1\geq b_2,\cdots\geq b_A.$ The above problem for $p=2$, is studied in \cite{anava2016k}. The approach we are going to solve the problem is as follows: a) Write Lagrangian b) Since the problem is convex, any solutions of KKT condition is global maximum. c) Obtain conditions using KKT conditions.

\begin{lemma}\label{LpwaterPouring}
Let $b \in\mathbb{R}^A$ be such that its components are in decreasing order (i,e $b_{i}\geq b_{i+1}$), $\alpha\geq 0$ be any non-negative constant, and 
\begin{equation}\label{app:eq:LpwaterPouringlemma}
 \zeta_p := \max_{c}-\alpha \lVert c\rVert_q + \langle c,b\rangle \qquad \text{such that }\qquad \sum_{i=1}^Ac_i = 1,\quad c_i\geq 0,\quad \forall i,    
\end{equation}
and let $c^*$ be a solution to the above problem.
 Then
\begin{enumerate}
    \item \label{app:wp:order} Higher components of $b$, gets higher weight in $c^*$. In other words, $c^*$ is also sorted component wise in descending order, that is \[c^*_1 \geq c^*_2,\cdots, \geq c^*_A.\] 
    \item \label{app:wp:zeta}The value $\zeta_p$ satisfies the following equation
    \[\alpha^{p} = \sum_{ b_i\geq \zeta_p}(b_i - \zeta_p)^{p}\]
    \item \label{app:wpl:policy} The solution $c$ of \eqref{app:eq:LpwaterPouringlemma}, is related to $\zeta_p$ as
    \begin{equation*}
        c_i = \frac{(b_i - \zeta_p)^{p-1}\mathbf{1}(b_i\geq\zeta_p)}{\sum_{s}(b_i - \zeta_p)^{p-1}\mathbf{1}(b_i\geq \zeta_p)}
    \end{equation*} 
    \item \label{app:wp:chi}Observe that the top $\chi_p:=\max\{i|b_i\geq\zeta_p\} $ actions are active and rest are passive. The number of active actions can be calculated as 
\[\{ k | \alpha^{p} \geq \sum_{i=1}^{k}(b_i - b_k)^{p}\} = \{1,2,\cdots,\chi_p\}.\]
\item Things can be re-written as
\[c_i  \propto \begin{cases}(b_i - \zeta_p)^{p-1} & \text{if}\quad i\leq \chi_p\\0 & \text{else}
\end{cases}\qquad \text{and}\qquad \alpha^{p} = \sum_{i=1}^{ \chi_p}(b_i - \zeta_p)^{p}\]

\item \label{app:wp:zetaBinarySearch}The function $\sum_{b_i \geq x}(b_i - x)^p $ is monotonically decreasing in $x$, hence the root $\zeta_p$ can be calculated efficiently by binary search between $[b_1 -\alpha, b_1]$.
    
\item \label{app:wp:zetaBound} Solution is sandwiched as follows
\[b_{\chi_p+1} \leq \zeta_p \leq b_{\chi_p}\]
\item \label{app:wp:chiSol}  $k\leq \chi_p$ if and only if there exist the solution of the following,
\[ \sum_{i=1}^k(b_i-x)^p = \alpha^p \quad \text{and}\quad x \leq b_k.\]

\item \label{app:wp:greedyInclusion} If action $k$ is active  and there is greedy increment hope then action $k+1$ is also active. That is 
    \[k \leq \chi_p \quad\text{and}\quad\lambda_k \leq b_{k+1} \implies k+1\leq \chi_p,\]
    where 
    \[ \sum_{i=1}^k(b_i-\lambda_k)^p = \alpha^p \quad \text{and}\quad \lambda_k \leq b_k.\]
\item \label{app:wp:stoppingCondition} If action $k$ is active, and there is no greedy hope and then action $k+1$ is not active. That is,\[k \leq \chi_p \quad\text{and}\quad\lambda_k > b_{k+1} \implies k+1> \chi_p,\]
    where 
    \[ \sum_{i=1}^k(b_i-\lambda_k)^p = \alpha^p \quad \text{and}\quad \lambda_k \leq b_k.\]
And this implies $k = \chi_p.$
\end{enumerate}
\end{lemma}
\begin{proof}
 \begin{enumerate}
    \item  Let \[f(c): = -\alpha\lVert c\rVert_q + \langle b,c\rangle.\]
    Let $c$ be any vector, and $c'$ be rearrangement $c$ in descending order. Precisely,
    \[c'_{k} := c_{i_k} , \quad \text{where}\quad c_{i_1}\geq c_{i_2},\cdots,\geq c_{i_A}.\]
    Then it is easy to see that $f(c')\geq f(c).$ And the claim follows.
    
     \item Writting Lagrangian of the optimization problem, and its derivative,
\begin{equation}\begin{aligned}
    &L = -\alpha \lVert c\rVert_q + \langle c,b\rangle + \lambda(\sum_{i}c_i -1) + \theta_ic_i\\
    &\frac{\partial L}{\partial c_i} = -\alpha \lVert c\rVert_q^{1-q}|c_i|^{q-2}c_i + b_i + \lambda + \theta_i,
\end{aligned}\end{equation}
$\lambda\in\mathbb{R}$ is multiplier for equality constraint $\sum_{i}c_i = 1$ and $\theta_1,\cdots,\theta_A \geq 0$ are multipliers for inequality constraints $c_i\geq 0,\quad \forall i\in [A].$ Using KKT (stationarity) condition, we have
\begin{equation}\label{app:eq:Lpwp:st}
     -\alpha \lVert c^*\rVert_q^{1-q}|c^*_i|^{q-2}c^*_i + b_i + \lambda + \theta_i = 0\\
\end{equation}
Let $\mathcal{B}:=\{i | c^*_i > 0\}$, then
\begin{equation}\begin{aligned}
    &\sum_{i\in\mathcal{B}} c^*_i[-\alpha \lVert c^*\rVert_q^{1-q}|c^*_i|^{q-2}c^*_i + b_i + \lambda ] = 0\\
    \implies &-\alpha \lVert c^*\rVert_q^{1-q}\lVert c^*\rVert^q_q + \langle c^*, b \rangle + \lambda  = 0, \qquad \text{(using $\sum_i c^*_i = 1$ and $(c^*_i)^2 = |c^*_i|^2$)}\\
     \implies &-\alpha \lVert c^*\rVert_q + \langle c^*, b \rangle + \lambda  = 0\\
     \implies &-\alpha \lVert c^*\rVert_q + \langle c^*, b \rangle  =- \lambda, \qquad \text{(re-arranging)} 
\end{aligned}\end{equation}

Now again using \eqref{app:eq:Lpwp:st},  we have 
\begin{equation}\begin{aligned}
   &-\alpha \lVert c^*\rVert_q^{1-q}|c^*_i|^{q-2}c^*_i + b_i + \lambda + \theta_i = 0\\
    \implies& \alpha\lVert c^*\rVert_q^{1-q}|c^*_i|^{q-2}c^*_i =  b_i + \lambda +\theta_i , \qquad \forall i, \qquad \text{(re-arranging)}
 \end{aligned}\end{equation}
 Now, if $i\in\mathcal{B}$ then $\theta_i = 0$ from complimentry slackness, so we have 
 \[\alpha\lVert c^*\rVert_q^{1-q}|c^*_i|^{q-2}c^*_i =  b_i + \lambda > 0 , \qquad \forall i\in\mathcal{B}\]
 by definition of $\mathcal{B}$. Now, if for some $i$, $b_i+\lambda > 0$ then $b_i+\lambda + \theta_i > 0$ as $\theta_i\geq 0$, that implies
 \[\alpha\lVert c^*\rVert_q^{1-q}|c^*_i|^{q-2}c^*_i =  b_i + \lambda +\theta_i > 0 \]
 \[\implies c^*_i > 0 \implies i\in\mathcal{B}. \]
 So, we have, 
 \[i\in\mathcal{B} \iff b_i+\lambda > 0.\]
  To summarize, we have 
\begin{equation}\label{app:eq:wp:c}
    \alpha\lVert c^*\rVert_q^{1-q}|c^*_i|^{q-2}c^*_i =  (b_i + \lambda) \mathbf{1}(b_i\geq-\lambda), \quad \forall i,\\
\end{equation}    
\begin{equation}\begin{aligned}
     \implies & \sum_{i}\alpha^{\frac{q}{q-1}}\lVert c^*\rVert_q^{-q}(c^*_i)^q = \sum_{i}(b_i + \lambda)^{\frac{q}{q-1}}\mathbf{1}(b_i\geq-\lambda), \\&\qquad \text{(taking $q/(q-1)$th power and summing)} \\
     \implies & \alpha^{p}  = \sum_{i=1}^A(b_i + \lambda)^{p}\mathbf{1}(b_i\geq-\lambda).
\end{aligned}\end{equation}
So, we have,
\begin{equation}\begin{aligned}\label{eq:waterpourng}
    \zeta_p    &= -\lambda \quad \text{such that }\quad \alpha^{p} = \sum_{b_i\geq \lambda}(b_i + \lambda)^{p}.\\
   \implies \alpha^{p} &= \sum_{ b_i\geq \zeta_p}(b_i - \zeta_p)^{p}
\end{aligned}\end{equation}
\item Furthermore, using \eqref{app:eq:wp:c}, we have
\begin{equation}\begin{aligned}
   &\alpha\lVert c^*\rVert_q^{1-q}|c^*_i|^{q-2}c^*_i =  (b_i + \lambda) \mathbf{1}(b_i\geq -\lambda) = (b_i -\zeta_p) \mathbf{1}(b_i\geq \zeta_p) \quad \forall i,\\
   \implies &  c^*_i \propto (b_i -\zeta_p)^{\frac{1}{q-1}}\mathbf{1}(b_i\geq \zeta_p)= \frac{(b_i - \zeta_p)^{p-1}\mathbf{1}(b_i\geq\zeta_p)}{\sum_{i}(b_i - \zeta_p)^{p-1}\mathbf{1}(b_i\geq \zeta_p)}, \qquad \text{(using $\sum_{i}c^*_i = 1$)}. \\
\end{aligned}\end{equation}

\item Now, we move on to calculate the number of active actions $\chi_p$. Observe that the function
\begin{equation}
    f(\lambda) := \sum_{i=1}^A(b_i-\lambda)^p\mathbf{1}(b_i\geq \lambda) - \alpha^p
\end{equation}
is monotonically decreasing in $\lambda$ and $\zeta_p$ is a root of $f$. This implies
\begin{equation}\begin{aligned}
    &f(x) \leq 0 \iff x \geq \zeta_p\\
    \implies & f(b_i) \leq 0 \iff b_i \geq \zeta_p\\
    \implies &\{i|b_i\geq \zeta_p\} =\{i|f(b_i)\leq 0\}\\
    \implies &\chi_p = \max\{i|b_i\geq \zeta_p\} = \max\{i|f(b_i)\leq 0\}.
\end{aligned}\end{equation}
Hence, things follows by putting back in the definition of $f$.
\item We have, 
\[\alpha^{p} = \sum_{i=1}^A(b_i - \zeta_p)^{p}\mathbf{1}(b_i\geq\zeta_p),\quad \text{and}\quad \chi_p = \max\{i|b_i\geq \zeta_p\}.\]
Combining both we have
\[\alpha^{p} = \sum_{i=1}^{\chi_p}(b_i - \zeta_p)^{p}.\]
And the other part follows directly.

\item Continuity and montonocity of the function $\sum_{b_i\geq x}(b_i - x)^{p} $ is trivial. Now observe that
$\sum_{b_i\geq b_1}(b_i - b_1)^{p} = 0$ and $ \sum_{b_i\geq b_1-\alpha}(b_i - (b_1-\alpha))^{p} \geq \alpha^p$,
so it implies that it is equal to $\alpha^p$ in the range $[b_1-\alpha,b_1]$.
\item Recall that the $\zeta_p$ is the solution to the following equation
\[\alpha^{p} = \sum_{b_i\geq x}(b_i - x)^{p}.\]
 And from the definition of $\chi_p$, we have \[\alpha^{p} < \sum_{i=1}^{ \chi_p +1}(b_i - b_{\chi_p +1})^{p} = \sum_{b_i\geq b_{\chi_p+1} }(b_i - b_{\chi_p+1 })^{p},\quad \text{and}\]
 \[ \alpha^{p} \geq \sum_{i=1}^{\chi_p} (b_i - b_{\chi_p })^{p} =\sum_{b_i\geq b_{\chi_p}} (b_i - b_{\chi_p })^{p}.\]
 
So from continuity, we infer the root $\zeta_p$ must lie between $[b_{\chi_p+1},b_{\chi}]$.
\item We  prove the first direction, and assume we have 
\begin{equation}\begin{aligned}
    & k \leq \chi_p \\
    \implies & \sum_{i=1}^k(b_i -b_k)^p \leq \alpha^p \qquad \text{(from definition of $\chi_p$)}.\\
\end{aligned}
\end{equation}
Observe the function $f(x):=\sum_{i=1}^k(b_i -x)^p$ is monotically decreasing in the range $(-\infty,b_k]$. Further, $f(b_k) \leq \alpha^p$  and $\lim_{x\to-\infty}f(x) = \infty$, so from the continuity argument there must exist a value $y\in (-\infty,b_k]$ such that $f(y) = \alpha^p$. This implies that 
\[\sum_{i=1}^k(b_i -y)^p \leq \alpha^p,\quad\text{and}\quad y\leq b_k.\]
Hence, explicitly showed the existence of the solution. Now, we move on to the second direction, and assume there exist $x$ such that
\[\sum_{i=1}^k(b_i -x)^p = \alpha^p,\quad\text{and}\quad x\leq b_k.\]
\[\implies \sum_{i=1}^k(b_i -b_k)^p \leq \alpha^p, \qquad \text{(as $x\leq b_k\leq b_{k-1}\cdots \leq b_1$)}\]
\[\implies k \leq \chi_p.\]

 \item We have $k \leq \chi_p$ and $\lambda_k$ such that 
 \begin{equation}\begin{split}
     \alpha^{p} 
     &= \sum_{i=1}^k(b_i - \lambda_k)^{p},\quad\text{and}\quad \lambda_k \leq b_k, \qquad \text{(from above item)}\\
      &  \geq \sum_{i=1}^k(b_i - b_{k+1})^{p},\qquad \text{(as $\lambda_k \leq b_{k+1} \leq b_k$)}\\
      & \geq \sum_{i=1}^{k+1}(b_i - b_{k+1})^{p},\qquad \text{(addition of $0$).}
 \end{split}
 \end{equation}
 From the definition of $\chi_p$, we get $k+1\leq \chi_p$. 
 
 \item We are given 
 \[ \sum_{i=1}^k(b_i-\lambda_k)^p = \alpha^p\]
 \[\implies \sum_{i=1}^k(b_i-b_{k+1})^p > \alpha^p,\qquad \text{(as $\lambda_k > b_{k+1}$)}\]
  \[\implies \sum_{i=1}^{k+1}(b_i-b_{k+1})^p > \alpha^p,\qquad \text{(addition of zero)}\]
    \[ \implies k+1 > \chi_p.\]

\end{enumerate}
\end{proof}

\subsubsection{Special case: L1}\label{app:L1waterpouringSC}
For $p=1$, by definition, we have
\begin{equation}\begin{aligned}
  \zeta_1& = \max_{c}-\alpha \lVert c\rVert_{\infty} + \langle c,b\rangle \qquad \text{such that }\qquad \sum_{a\in \mathcal{A}}c_a = 1,\quad c\succeq 0.\\
\end{aligned}\end{equation}
And $\chi_1$ is the optimal number of actions, that is 
\[\alpha = \sum_{i=1}^{\chi_1}(b_i - \zeta_1) \]
\[\implies \zeta_1 =\frac{\sum_{i=1}^{\chi_1} b_i - \alpha}{\chi_1}. \]
Let $\lambda_k$ be the such that  
\[\alpha = \sum_{i=1}^{k}(b_i - \lambda_k) \]
\[\implies \lambda_k =\frac{\sum_{i=1}^{k} b_i - \alpha}{k}. \]

\begin{proposition}
 \[\zeta_1 = \max_{k}\lambda_k\]
\end{proposition}
 \begin{proof}
 From lemma \ref{LpwaterPouring}, we have 
 \[\lambda_1 \leq \lambda_2 \cdots \leq\lambda_{\chi_1}.\]
Now, we have  
\begin{equation}\begin{split}\label{eq:Telelambda}
    \lambda_{k}-\lambda_{k+m} &= \frac{\sum_{i=1}^kb_i- \alpha}{k}-\frac{\sum_{i=1}^{k+m}b_i- \alpha}{k+m}\\ 
    &= \frac{\sum_{i=1}^kb_i- \alpha}{k}-\frac{\sum_{i=1}^{k}b_i- \alpha}{k+m} -\frac{\sum_{i=1}^m b_{k+i}}{k+m}\\ 
    &= \frac{m(\sum_{i=1}^kb_i- \alpha}{k(k+m))} -\frac{\sum_{i=1}^m b_{k+i}}{k+m}\\
    &= \frac{m}{k+m}(\frac{\sum_{i=1}^kb_i- \alpha}{k} -\frac{\sum_{i=1}^m b_{k+i}}{m})\\
    &= \frac{m}{k+m}(\lambda_k -\frac{\sum_{i=1}^m b_{k+i}}{m})\\
\end{split}
\end{equation}
 From lemma \ref{LpwaterPouring}, we also know the stopping criteria for $\chi_1$, that is 
 \[\lambda_{\chi_1} > b_{\chi_1 +1}\]
 \[\implies \lambda_{\chi_1} > b_{\chi_1 + i}, \qquad  i\geq 1,\qquad \text{(as $b_i$ are in descending order)}  \]
 \[\implies \lambda_{\chi_1} > \frac{\sum_{i=1}^m b_{\chi_1+i}}{m}, \qquad \forall m\geq 1.\]
 Combining it with the \eqref{eq:Telelambda}, for all $m
 \geq 0$ , we get 
 \begin{equation}\begin{split}
 \lambda_{\chi_1}-\lambda_{\chi_1+m} &= \frac{m}{\chi_1+m}(\lambda_{\chi_1} -\frac{\sum_{i=1}^m b_{\chi_1+i}}{m})\\
 &\geq 0\\
 \implies \lambda_{\chi_1}&\geq \lambda_{\chi_1+m} 
 \end{split}
 \end{equation}
 Hence, we get the desired result, 
\[\zeta_1 = \lambda_{\chi_1} = \max_{k}\lambda_k.\]
 \end{proof} 

\subsubsection{Special case: max norm}
For $p=\infty$, by definition, we have
\begin{equation}\begin{aligned}
  \zeta_\infty(b)&= \max_{c}-\alpha \lVert c\rVert_{1} + \langle c,b\rangle \qquad \text{such that }\qquad \sum_{a\in \mathcal{A}}c_a = 1,\quad c\succeq 0.\\
= &\max_{c}-\alpha + \langle c,b\rangle \qquad \text{such that }\qquad \sum_{a\in \mathcal{A}}c_a = 1,\quad c\succeq 0. \\
= & -\alpha + \max_{i}b_i
\end{aligned}\end{equation}

\subsubsection{Special case: L2}
The problem is discussed in great details in \cite{anava2016k}, here we outline the proof.
For $p=2$, we have
\begin{equation}\begin{aligned}
  \zeta_2&= \max_{c}-\alpha \lVert c\rVert_{2} + \langle c,b\rangle \qquad \text{such that }\qquad \sum_{a\in \mathcal{A}}c_a = 1,\quad c\succeq 0.\\
\end{aligned}\end{equation}
Let $\lambda_k$ be the solution of the following equation
\begin{equation}\begin{split}
\alpha^{2} &= \sum_{i=1}^k(b_i - \lambda)^{2},\qquad \lambda \leq b_k\\
    &=  k\lambda^2 -2\sum_{i=1}^k\lambda b_i. +\sum_{i=1}^k(b_i)^2, \qquad \lambda \leq b_k \\
 \implies \lambda_k &= \frac{\sum_{i=1}^k b_i \pm \sqrt{(\sum_{i=1}^kb_i)^2 - k(\sum_{i=1}^k(b_i)^2 - \alpha^{2} )}}{k}, \quad \text{and } \quad \qquad \lambda_k \leq b_k \\
  &= \frac{\sum_{i=1}^k b_i - \sqrt{(\sum_{i=1}^kb_i)^2 - k(\sum_{i=1}^k(b_i)^2 - \alpha^{2} )}}{k}\\
  &= \frac{\sum_{i=1}^k b_i}{k} - \sqrt{\alpha^{2} -\sum_{i=1}^k(b_i -\frac{\sum_{i=1}^kb_i}{k})^2}\\
\end{split}
\end{equation}
From lemma \ref{LpwaterPouring}, we know
 \[\lambda_1 \leq \lambda_2 \cdots \leq \lambda_{\chi_2} = \zeta_2\]
 where $\chi_2$ calculated in two ways: a) 
 \[\chi_2 = \max_{m}\{m| \sum_{i=1}^m(b_i-b_m)^2\leq \alpha^2\}\]
 b) \[\chi_2 = \min_{m}\{m |\lambda_{m}\leq b_{m+1}\}\]
 We proceed greedily until stopping condition is met in lemma \ref{LpwaterPouring}. Concretely, it is illustrated in algorithm \ref{alg:f2}.

\subsection{L1 Water Pouring lemma}\label{app:L1waterpouringLemma}
In this section, we re-derive the above water pouring lemma for 
$p=1$ from scratch, just for sanity check. As in the above proof, there is a possibility of some breakdown, as we had take limits $q\to\infty$. We will see that all the above results for $p=1$ too.

Let $b \in\mathbb{R}^A$ be such that its components are in decreasing order, i,e $b_{i}\geq b_{i+1}$ and
\begin{equation}\label{eq:LpwaterPouringlemma}
 \zeta_1 := \max_{c}-\alpha \lVert c\rVert_\infty + \langle c,b\rangle \qquad \text{such that }\qquad \sum_{i=1}^Ac_i = 1,\quad c_i\geq 0,\quad \forall i.   
\end{equation}
 Lets fix any vector $c\in\mathbb{R}^A$, and let $k_1 :=\lfloor\frac{1}{\max_{i}c_i}\rfloor$ and let \[c^1_i = \begin{cases}\max_{i}c_i\qquad&\text{if}\quad i\leq k_1\\
1- k_1\max_{i}c_i\qquad&\text{if}\quad i= k_1+1\\
0\qquad&\text{else}\\\end{cases}  \]
Then we have,
 \begin{equation}\begin{aligned}
     -\alpha \lVert c\rVert_\infty + \langle c,b\rangle =& -\alpha \max_{i}c_i + \sum_{i=1}^A c_ib_i\\
     \leq&-\alpha \max_{i}c_i + \sum_{i=1}^A c^1_ib_i, \qquad\text{(recall $b_i$ is in decreasing order)}\\
     =& -\alpha \lVert c^1\rVert_\infty + \langle c^1,b\rangle
 \end{aligned}\end{equation}
  Now, lets define $c^2\in\mathbb{R}^A$. Let 
  \[k_2 = \begin{cases}k_1+1\qquad&\text{if}\quad \frac{\sum_{i=1}^{k_1}b_i -\alpha}{k_1} \leq b_{k+1}\\
k_1\qquad &\text{else}\\\end{cases}  \]
and let $c^2_i = \frac{\mathbf{1}(i\leq k_2)}{k_2}$. Then we have,
\begin{equation}\begin{aligned}
     -\alpha \lVert c^1\rVert_\infty + \langle c^1,b\rangle =&-\alpha \max_{i}c_i + \sum_{i=1}^A c^1_ib_i\\
     =&-\alpha \max_{i}c_i + \sum_{i=1}^{k_1}\max_{i}c_i b_i +(1-k_1\max_{i}c_i) b_{k_1 +1},\quad\text{(definition of $c^1$)}\\
    =&(\frac{-\alpha  +  \sum_{i=1}^{k_1}b_i}{k_1})k_1\max_{i}c_i + b_{k_1 +1}(1-k_1\max_{i}c_i),\qquad \text{(re-arranging)}\\
    \leq & \frac{-\alpha  +  \sum_{i=1}^{k_2}b_i}{k_2}\\
     =& -\alpha \lVert c^2\rVert_\infty + \langle c^2,b\rangle
 \end{aligned}\end{equation}
 The last inequality comes from the definition of $k_2$ and $c^2$. So we conclude that a optimal solution is uniform over some actions, that is
 \begin{equation}\begin{split}
     \zeta_1 =& \max_{c\in\mathcal{C}}-\alpha \lVert c\rVert_\infty + \langle c,b\rangle\\
     =& \max_{k}\bigm(\frac{-\alpha+\sum_{i=1}^kb_i}{k}\bigm)\\
 \end{split}
 \end{equation}
 where $\mathcal{C}:=\{c^k\in\mathbb{R}^A|c^k_i = \frac{\mathbf{1}(i\leq k)}{k} \}$ is set of uniform actions. Rest all the properties follows same as $L_p$ water pouring lemma.

\section{Robust Value Iteration (Main)} \label{app:srLp}
In this section, we will discuss the main results from the paper except for time complexity results. It contains the proofs of the results presented in the main body and also some other corollaries/special cases.

\subsection{sa-rectangular robust policy evaluation and improvement}
 \begin{theorem} $(\mathtt{sa})$-rectangular $L_p$ robust Bellman operator is equivalent to reward regularized (non-robust) Bellman operator, that is 
\begin{equation*}\begin{aligned}
    (\mathcal{T}^\pi_{\mathcal{U}^{\mathtt{sa}}_p} v)(s)  =& \sum_{a}\pi(a|s)[  -\alpha_{s,a} -\gamma\beta_{s,a}\kappa_q(v)  +R_0(s,a) +\gamma \sum_{s'}P_0(s'|s,a)v(s')], \qquad \text{and}\\
    (\mathcal{T}^*_{\mathcal{U}^{\mathtt{sa}}_p} v)(s)  =& \max_{a\in\mathcal{A}}[  -\alpha_{s,a} -\gamma\beta_{s,a}\kappa_q(v)  +R_0(s,a) +\gamma \sum_{s'}P_0(s'|s,a)v(s')],
\end{aligned}\end{equation*}
where $\kappa_p$ is defined in \eqref{def:kp}.
\end{theorem}
\begin{proof}
From definition robust Bellman operator and  $\mathcal{U}^{\mathtt{sa}}_p = (R_0 +\mathcal{R})\times(P_0 +\mathcal{P})$, we have,
\begin{equation}\begin{aligned}
   (&\mathcal{T}^\pi_{\mathcal{U}^{\mathtt{sa}}_p} v)(s)=\min_{{R,P\in\mathcal{U}^{\mathtt{sa}}_p}}\sum_{a}\pi(a|s)\Bigm[R(s,a) + \gamma \sum_{s'}P(s'|s,a)v(s')\Bigm] \\
    &=\sum_{a}\pi(a|s)\Bigm[R_0(s,a) + \gamma \sum_{s'}P_0(s'|s,a)v(s')\Bigm]+   \\
    &\qquad\qquad \min_{{p\in\mathcal{P}},r\in\mathcal{R}}\sum_{a}\pi(a|s)\Bigm[r(s,a) +\gamma \sum_{s'}p(s'|s,a)v(s')\Bigm] ,\\
    & \quad \text{(from $(\mathtt{sa})$-rectangularity, we get)}\\
    &=\sum_{a}\pi(a|s)\Bigm[R_0(s,a) + \gamma \sum_{s'}P_0(s'|s,a)v(s')\Bigm]+ \\
    &\qquad\qquad \sum_{a}\pi(a|s)\underbrace{\min_{{p_{s,a}\in\mathcal{P}_{sa}},r_{s,a}\in\mathcal{R}_{s,a}}\Bigm[r_{s,a} + \gamma \sum_{s'}p_{s,a}(s')v(s')\Bigm]}_{:=\Omega_{sa}(v)} 
\end{aligned}\end{equation}
Now we focus on regularizer function $\Omega$, as follows
\begin{equation}\begin{aligned}
    \Omega_{sa}(v)=&\min_{{p_{s,a}\in\mathcal{P}_{s,a}},r_{s,a}\in\mathcal{R}_{s,a}}\Bigm[r_{s,a} + \gamma \sum_{s'}p_{s,a}(s')v(s')\Bigm] \\
    =&\min_{r_{s,a}\in\mathcal{R}_{s,a}}r_{s,a} + \gamma\min_{{p_{s,a}\in\mathcal{P}_{sa}}} \sum_{s'}p_{s,a}(s')v(s') \\
    &=- \alpha_{s,a} +\gamma\min_{\lVert p_{sa}\rVert_p\leq \beta_{s,a},\sum_{s'}p_{sa}(s')=0} \langle p_{s,a}, v\rangle,\\
    =&- \alpha_{s,a} -\gamma \beta_{s,a}\kappa_q(v), \qquad \text{(from lemma \ref{regfn}).}\\
\end{aligned}\end{equation}
Putting back, we have
\begin{equation*}
    (\mathcal{T}^\pi_{\mathcal{U}^{\mathtt{sa}}_p} v)(s)=\sum_{a}\pi(a|s)\Bigm[- \alpha_{s,a} -\gamma \beta_{s,a}\kappa_q(v)+R_0(s,a) + \gamma \sum_{s'}P_0(s'|s,a)v(s')\Bigm]
\end{equation*}
Again, reusing above results in optimal robust operator, we have
\begin{equation}\begin{aligned}
    (\mathcal{T}^*_{\mathcal{U}^{\mathtt{sa}}_p} v)(s) &= \max_{\pi_s\in\Delta_\mathcal{A}}\min_{{R,P\in\mathcal{U}^{\mathtt{sa}}_p}}\sum_{a}\pi_s(a)\Bigm[R(s,a) + \gamma \sum_{s'}P(s'|s,a)v(s')\Bigm]\\
    &=\max_{\pi_s\in\Delta_\mathcal{A}}\sum_{a}\pi_s(a)\Bigm[- \alpha_{s,a} -\gamma \beta_{s,a}\kappa_p(v)+R_0(s,a) + \gamma \sum_{s'}P_0(s'|s,a)v(s')\Bigm]\\
    &=\max_{a\in\mathcal{A}}\Bigm[- \alpha_{s,a} -\gamma \beta_{s,a}\kappa_q(v)+R_0(s,a) + \gamma \sum_{s'}P_0(s'|s,a)v(s')\Bigm]
\end{aligned}\end{equation}
The claim is proved.
\end{proof}

\subsection{S-rectangular robust policy evaluation}\label{app:sLprpe}
\begin{theorem} $\mathtt{S}$-rectangular $L_p$ robust Bellman operator is equivalent to reward regularized (non-robust) Bellman operator, that is 
\begin{equation*}
    (\mathcal{T}^\pi_{\mathcal{U}^s_p} v)(s)  =   -\Bigm(\alpha_s +\gamma\beta_{s}\kappa_q(v)\Bigm)\lVert\pi(\cdot|s)\rVert_q  +\sum_{a}\pi(a|s)\Bigm(R_0(s,a) +\gamma \sum_{s'}P_0(s'|s,a)v(s')\Bigm)
\end{equation*}
where $\kappa_p$ is defined in \eqref{def:kp} and $\lVert \pi(\cdot|s)\rVert_q$ is $q$-norm of the vector $\pi(\cdot|s)\in\Delta_{\mathcal{A}}$.
\end{theorem}

\begin{proof}
From definition of robust Bellman operator and  $\mathcal{U}^{\mathtt{s}}_p = (R_0 +\mathcal{R})\times(P_0 +\mathcal{P})$, we have
\begin{equation}\begin{aligned}
    (&\mathcal{T}^\pi_{\mathcal{U}^s_p} v)(s)  = 
    \min_{{R,P\in\mathcal{U}^s_p}}\sum_{a}\pi(a|s)\Bigm[R(s,a) + \gamma \sum_{s'}P(s'|s,a)v(s')\Bigm] \\
    &=\sum_{a}\pi(a|s)\Bigm[\underbrace{R_0(s,a) + \gamma \sum_{s'}P_0(s'|s,a)v(s')}_{\text{nominal values}}\Bigm] \\
    &\qquad\qquad +\min_{{p\in\mathcal{P}},r\in\mathcal{R}}\sum_{a}\pi(a|s)\Bigm[r(s,a) + \gamma \sum_{s'}p(s'|s,a)v(s')\Bigm]\\
    &\qquad \text{(from $\mathtt{s}$-rectangularity we have)}\\
    &=\sum_{a}\pi(a|s)\Bigm[R_0(s,a) + \gamma \sum_{s'}P_0(s'|s,a)v(s')\Bigm] \\
    & \qquad\qquad +  \underbrace{\min_{{p_s\in\mathcal{P}_s},r_s\in\mathcal{R}_s}\sum_{a}\pi(a|s)\Bigm[r_s(a) + \gamma \sum_{s'}p_s(s'|a)v(s')\Bigm]}_{:=\Omega_s(\pi_s,v)}
\end{aligned}\end{equation}
where we denote $\pi_s(a) = \pi(a|s)$ as a shorthand. Now we calculate the regularizer function as follows

\begin{equation}\begin{aligned}
    \Omega_s(\pi_s,v):=&\min_{r_s\in\mathcal{R}_s,p_s\in\mathcal{P}_s}\langle r_s + \gamma v^T p_s,\pi_s \rangle =\min_{r_s\in\mathcal{R}_s}\langle r_s,\pi_s\rangle + \gamma\min_{p_s\in\mathcal{P}_s} v^T p_s\pi_s   \\
    &=-\alpha_s\lVert \pi_s\rVert_q +\gamma\min_{p_s\in\mathcal{P}_s} v^T p_s\pi_s, \qquad \text{(using  $\frac{1}{p} + \frac{1}{q} = 1$  )}\\
    =&- \alpha_s\lVert \pi_s\rVert_q +\gamma \min_{p_s\in\mathcal{P}_s}\sum_{a}\pi_s(a)\langle p_{s,a}, v\rangle\\
    =&- \alpha_s\lVert \pi_s\rVert_q +\gamma \min_{\sum_{a}(\beta_{s,a})^p \leq (\beta_s)^p}\quad \min_{\lVert p_{sa}\rVert_p\leq \beta_{s,a},\sum_{s'}p_{sa}(s')=0 }\quad\sum_{a}\pi_s(a)\langle p_{s,a}, v\rangle  \\
    =&- \alpha_s\lVert \pi_s\rVert_q +\gamma \min_{\sum_{a}(\beta_{s,a})^p \leq (\beta_s)^p}\sum_{a}\pi_s(a)\quad \min_{\lVert p_{sa}\rVert_p\leq \beta_{s,a},\sum_{s'}p_{sa}(s')=0 }\quad\langle p_{s,a}, v\rangle  \qquad \text{}\\
    =&- \alpha_s\lVert \pi_s\rVert_q +\gamma \min_{\sum_{a}(\beta_{sa})^p \leq (\beta_s)^p}\sum_{a}\pi_s(a)(-\beta_{sa}\kappa_p(v))  \qquad \text{ ( from lemma \ref{regfn})}\\
     =&- \alpha_s\lVert \pi_s\rVert_q -\gamma \kappa_q(v)\max_{\sum_{a}(\beta_{sa})^p \leq (\beta_s)^p}\sum_{a}\pi_s(a)\beta_{sa}  \qquad \text{}\\
      =&- \alpha_s\lVert \pi_s\rVert_q -\gamma \kappa_p(v)\lVert \pi_s\rVert_q\beta_{s}  \qquad \text{(using Holders)}\\
      =&- (\alpha_s +\gamma\beta_{s}\kappa_q(v))\lVert \pi_s\rVert_q .  \\
\end{aligned}\end{equation}
Now putting above values in robust operator, we have
\begin{equation*}\begin{aligned}
    (\mathcal{T}^\pi_{\mathcal{U}^s_p} v)(s) 
     &=- \Bigm(\alpha_s +\gamma\beta_{s}\kappa_q(v)\Bigm)\lVert \pi(\cdot|s)\rVert_q +\\ &\sum_{a}\pi(a|s)\Bigm(R_0(s,a) + \gamma \sum_{s'}P_0(s'|s,a)v(s')\Bigm).
\end{aligned}\end{equation*}
\end{proof}

\subsection{s-rectangular robust policy improvement}

Reusing robust policy evaluation results in section \ref{app:sLprpe}, we have
\begin{equation}\begin{aligned}
    (\mathcal{T}^*_{\mathcal{U}^{\mathtt{s}}_p} v)(s) &= \max_{\pi_s\in\Delta_\mathcal{A}}\min_{{R,P\in\mathcal{U}^{\mathtt{sa}}_p}}\sum_{a}\pi_s(a)\Bigm[R(s,a) + \gamma \sum_{s'}P(s'|s,a)v(s')\Bigm]\\
    &=\max_{\pi_s\in\Delta_{\mathcal{A}}}\Bigm[  -(\alpha_s +\gamma\beta_{s}\kappa_q(v))\lVert \pi_s\rVert_q  +\sum_{a}\pi_s(a)(R(s,a)  + \gamma \sum_{s'}P(s'|s,a) v(s'))\Bigm].
\end{aligned}\end{equation}
Observe that, we have the following form
\begin{equation}\label{app:eq:wp:ref}
(\mathcal{T}^*_{\mathcal{U}^{\mathtt{s}}_p} v)(s)= \max_{c}-\alpha \lVert c\rVert_q + \langle c,b\rangle \qquad \text{such that }\qquad \sum_{i=1}^A c_i = 1,\quad c\succeq 0,
\end{equation}
where $\alpha =\alpha_s +\gamma\beta_{s}\kappa_q(v) $ and $b_i =R(s,a_i)  + \gamma \sum_{s'}P(s'|s,a_i) v(s') $. Now all the results below, follows from water pouring lemma ( lemma \ref{LpwaterPouring}).

\begin{theorem}(Policy improvement) The optimal robust Bellman operator can be evaluated in following ways.
\begin{enumerate}
    \item $(\mathcal{T}^*_{\mathcal{U}^s_p}v)(s)$ is the solution of the following equation that can be found using binary search between $\bigm[\max_{a}Q(s,a)-\sigma, \max_{a}Q(s,a)\bigm]$,

\begin{equation}
    \sum_{a}\bigm(Q(s,a) - x\bigm)^p\mathbf{1}\bigm( Q(s,a) \geq x\bigm)  = \sigma^p.
\end{equation}
\item $(\mathcal{T}^*_{\mathcal{U}^s_p}v)(s)$ and $\chi_p(v,s)$ can also be computed through algorithm \ref{alg:fp}.
\end{enumerate}
where $\sigma = \alpha_s + \gamma\beta_s\kappa_q(v),$ and $Q(s,a)= R_0(s,a) + \gamma\sum_{s'} P_0(s'|s,a)v(s')$. \end{theorem}
\begin{proof}
The first part follows from lemma \ref{LpwaterPouring}, point \ref{app:wp:zeta}. The second part follows from lemma \ref{LpwaterPouring}, point \ref{app:wp:greedyInclusion} (greedy inclusion ) and point \ref{app:wp:stoppingCondition} (stopping condition).
\end{proof}

\begin{theorem}(Go To Policy)\label{rs:GoToPolicy} The greedy policy $\pi$ w.r.t. value function $v$, defined as $\mathcal{T}^*_{\mathcal{U}^s_p}v =\mathcal{T}^\pi_{\mathcal{U}^s_p}v$ is a threshold policy. It takes only those actions that has positive advantage, with probability proportional to $(p-1)$th power of its advantage. That is
\[\pi(a|s)\propto (A(s,a))^{p-1}\mathbf{1}(A(s,a)\geq 0),\]

where $A(s,a)= R_0(s,a) + \gamma\sum_{s'} P_0(s'|s,a)v(s') -  (\mathcal{T}^*_{\mathcal{U}^s_p}v)(s)$.
\end{theorem}
\begin{proof}
Follows from lemma \ref{LpwaterPouring}, point \ref{app:wpl:policy}.
\end{proof}

\begin{property} $\chi_p(v,s)$ is number of actions that has positive advantage, that is 
\[\chi_p(v,s) = \Bigm\lvert\bigm\{a \mid (\mathcal{T}^*_{\mathcal{U}^s_p}v)(s) \leq  R_0(s,a) + \gamma\sum_{s'} P_0(s'|s,a)v(s')\bigm\}\Bigm\rvert.\]
\end{property}
\begin{proof}

Follows from lemma \ref{LpwaterPouring}, point \ref{app:wp:chi}.
\end{proof}

\begin{property}( Value vs Q-value) $(\mathcal{T}^*_{\mathcal{U}^s_p}v)(s)$ is bounded by the Q-value of $\chi$th and $(\chi+1)$th actions. That is  
\[Q(s, a_{\chi+1}) < (\mathcal{T}^*_{\mathcal{U}^s_p}v)(s) \leq Q(s, a_\chi),\qquad\text{ where}\quad \chi = \chi_p(v,s),\]
 $Q(s,a) =R_0(s,a) +\gamma \sum_{s'}P_0(s'|s,a)v(s')$, and $Q(s,a_1)\geq Q(s,a_2),\cdots Q(s,a_A)$.
\end{property}
\begin{proof}
Follows from lemma \ref{LpwaterPouring}, point \ref{app:wp:zetaBound}.
\end{proof}

\begin{corollary} For $p=1$, the optimal policy $\pi_1$ w.r.t. value function $v$ and uncertainty set $\mathcal{U}^s_1$, can be computed directly using $\chi_1(s)$ without calculating advantage function. That is 
\[\pi_1(a^s_i|s) = \frac{\mathbf{1}(i\leq \chi_1(s))}{\chi_1(s)}.\]
\end{corollary}
\begin{proof}
Follows from Theorem \ref{rs:GoToPolicy} by putting $p=1$. Note that it can be directly obtained using $L_1$ water pouring lemma (see section \ref{app:L1waterpouringLemma})
\end{proof}

\begin{corollary}\label{rs:policy:inf} (For $p=\infty$) The optimal policy $\pi$ w.r.t. value function $v$ and uncertainty set $\mathcal{U}^s_\infty$   (precisely $\mathcal{T}^*_{\mathcal{U}^s_\infty}v = \mathcal{T}^\pi_{\mathcal{U}^s_\infty}v$), is to play the best response, that is 
\[\pi(a|s) = \frac{\mathbf{1}(a\in \text{arg}\max_{a}Q(s,a))}{\bigm \lvert \text{arg}\max_{a}Q(s,a)\bigm\rvert}.\]
In case of tie in the best response, it is optimal to play any of the best responses with any probability. 
\end{corollary}
\begin{proof}
Follows from Theorem \ref{rs:GoToPolicy} by taking limit $p\to\infty$. \end{proof}

\begin{corollary} For $p=\infty$, $\mathcal{T}^*_{\mathcal{U}^s_p}v$, the robust optimal Bellman operator evaluation can be obtained in closed form. That is 
\[(\mathcal{T}^*_{\mathcal{U}^s_\infty}v)(s) = \max_{a}Q(s,a) - \sigma,\]
where $\sigma = \alpha_s + \gamma\beta_s\kappa_1(v), Q(s,a) = R_0(s,a) + \gamma\sum_{s'} P_0(s'|s,a)v(s')$.
\end{corollary}
\begin{proof}
Let $\pi$ be such that
\[\mathcal{T}^*_{\mathcal{U}^s_\infty}v = \mathcal{T}^\pi_{\mathcal{U}^s_\infty}v.\]
This implies
\begin{equation}\begin{aligned}
    (\mathcal{T}^*_{\mathcal{U}^{\mathtt{s}}_p} v)(s) &= \min_{{R,P\in\mathcal{U}^{\mathtt{sa}}_p}}\sum_{a}\pi(a|s)\Bigm[R(s,a) + \gamma \sum_{s'}P(s'|s,a)v(s')\Bigm]\\
    &=  -(\alpha_s +\gamma\beta_{s}\kappa_p(v))\lVert \pi(\cdot|s)\rVert_q  +\sum_{a}\pi(a|s)(R(s,a)  + \gamma \sum_{s'}P(s'|s,a) v(s')).
\end{aligned}\end{equation}
From corollary \ref{rs:policy:inf}, we know the that $\pi$ is deterministic best response policy. Putting this we get the desired result.\\
There is a another way of proving this, using Theorem \ref{rs:rve} by taking limit $p\to \infty$ carefully as
\begin{equation}
    \lim_{p\to\infty}\sum_{a}\left(Q(s,a) - \mathcal{T}^*_{\mathcal{U}^{\mathtt{s}}_p} v)(s)\right)^p\mathbf{1}\left( Q(s,a) \geq \mathcal{T}^*_{\mathcal{U}^{\mathtt{s}}_p} v)(s)\right))^\frac{1}{p} = \sigma,
\end{equation}
where $\sigma = \alpha_s + \gamma\beta_s\kappa_1(v)$.
\end{proof}

\begin{corollary} For $p=1$, the robust optimal Bellman operator $\mathcal{T}^*_{\mathcal{U}^s_p}$, can be computed in closed form. That is
\[(\mathcal{T}^*_{\mathcal{U}^s_p}v)(s) = \max_{k}\frac{\sum_{i=1}^{k} Q(s,a_i) - \sigma}{k},\]
where $\sigma = \alpha_s + \gamma\beta_s\kappa_\infty(v), Q(s,a) = R_0(s,a) + \gamma\sum_{s'} P_0(s'|s,a)v(s')$, and $Q(s,a_1)\geq Q(s,a_2),\geq \cdots \geq Q(s,a_A).$\end{corollary}
\begin{proof}
Follows from section \ref{app:L1waterpouringSC}. 
\end{proof}

\begin{corollary} \label{rs:f1f2}
The $\mathtt{s}$ rectangular $L_p$ robust Bellman operator can be evaluated for $p =1 ,2$ by algorithm \ref{alg:f1} and algorithm \ref{alg:f2} respectively.
\end{corollary}
\begin{proof}
It follows from the algorithm \ref{alg:fp}, where we solve the linear equation and quadratic equation for $p=1,2$ respectively. For $p=2$, it can be found in \cite{anava2016k}.
\end{proof}

\begin{algorithm}
\caption{Algorithm to compute $S$-rectangular $L_2$ robust optimal Bellman Operator}\label{alg:f2}
\begin{algorithmic} [1]
 \STATE \textbf{Input:} $\sigma = \alpha_s +\gamma\beta_s\kappa_2(v), \qquad Q(s,a) = R_0(s,a) + \gamma\sum_{s'} P_0(s'|s,a)v(s')$.
 \STATE \textbf{Output} $(\mathcal{T}^*_{\mathcal{U}^s_2}v)(s), \chi_2(v,s)$
\STATE Sort $Q(s,\cdot)$ and label actions such that $Q(s,a_1)\geq Q(s,a_2), \cdots$.
\STATE Set initial value guess $\lambda_1 = Q(s,a_1)-\sigma$ and counter $k=1$.
\WHILE{$k \leq A-1  $ and $\lambda_k \leq Q(s,a_k)$}
    \STATE Increment counter: $k = k+1$
    \STATE Update value estimate: \[\lambda_k = \frac{1}{k}\Bigm[\sum_{i=1}^{k}Q(s,a_i) - \sqrt{k\sigma^2 + (\sum_{i=1}^{k}Q(s,a_i))^2 - k\sum_{i=1}^{k}(Q(s,a_i))^2}\Bigm]\]
\ENDWHILE
\STATE Return: $\lambda_k, k $
\end{algorithmic}
\end{algorithm}

\begin{algorithm}
\caption{Algorithm to compute $S$-rectangular $L_1$ robust optimal Bellman Operator}\label{alg:f1}
\begin{algorithmic} [1]
 \STATE \textbf{Input:} $\sigma = \alpha_s +\gamma\beta_s\kappa_\infty(v), \qquad Q(s,a) = R_0(s,a) + \gamma\sum_{s'} P_0(s'|s,a)v(s')$.
 \STATE \textbf{Output} $(\mathcal{T}^*_{\mathcal{U}^s_1}v)(s), \chi_1(v,s)$
\STATE Sort $Q(s,\cdot)$ and label actions such that $Q(s,a_1)\geq Q(s,a_2), \cdots$.
\STATE Set initial value guess $\lambda_1 = Q(s,a_1)-\sigma$ and counter $k=1$.
\WHILE{$k \leq A-1  $ and $\lambda_k \leq Q(s,a_k)$}
    \STATE Increment counter: $k = k+1$
    \STATE Update value estimate: \[\lambda_k = \frac{1}{k}\Bigm[\sum_{i=1}^kQ(s,a_i) -\sigma \Bigm]\]
\ENDWHILE
\STATE Return: $\lambda_k, k $
\end{algorithmic}
\end{algorithm}

\section{Time Complexity} \label{app:timeComplexitySection}

In this section, we will discuss time complexity of various robust MDPs and compare it with time complexity of non-robust MDPs. We assume that we have the  knowledge of nominal transition kernel and nominal reward function for robust MDPs, and in case of non-robust MDPs, we assume the knowledge of the transition kernel and reward function. We divide the discussion into various parts depending upon their similarity. 

\subsection{Exact Value Iteration: Best Response}
In this section, we will discuss non-robust MDPs, $(\mathtt{sa})$-rectangular $L_1/L_2/L_\infty$ robust MDPs and $\mathtt{s}$-rectangular $L_\infty$ robust MDPs. They all have  a common theme for value iteration as follows, for the value function $v$, their Bellman operator ( $\mathcal{T}$) evaluation is done as
\begin{equation}\begin{aligned}
   (\mathcal{T} v)(s) =& \underbrace{\max_{a}}_{\text{action cost}}\Bigm[R(s,a)+ \alpha_{s,a}\underbrace{\kappa(v)}_{\text{reward penalty/cost}} + \gamma \underbrace{\sum_{s'}P(s'|s,a)v(s')}_{\text{sweep}}\Bigm]. 
\end{aligned}\end{equation}
'Sweep' requires $O(S)$ iterations and 'action cost' requires $O(A)$ iterations. Note that the reward penalty $\kappa(v)$ doesn't depend on state and action. It is calculated only once for value iteration for all states. The above value update has to be done for each states , so one full update requires
 \[ O\Bigm(S(\text{action cost}) (\text{sweep cost} \bigm) +  \text{reward cost}\Bigm)= O\Bigm(S^2A  +  \text{reward cost}\Bigm)\]
Since the value iteration is a contraction map, so to get $\epsilon$-close to the optimal value, it requires $O(\log(\frac{1}{\epsilon}))$ full value update, so the complexity is  
\[O\Bigm(\log(\frac{1}{\epsilon})\bigm(S^2A + \text{reward cost}\bigm)\Bigm).\]

\begin{enumerate}
    \item \textbf{Non-robust MDPs}: The cost of 'reward is zero as there is no regularizer to compute. The total complexity is   
     \[O\Bigm(\log(\frac{1}{\epsilon})\bigm(S^2A + 0 \bigm)\Bigm) = O\Bigm(\log(\frac{1}{\epsilon})S^2A\Bigm).\]
     
    \item  \textbf{$(\mathtt{sa})$-rectangular $L_1/L_2/L_\infty$ and $\mathtt{s}$-rectangular $L_\infty$ robust MDPs}: We need to calculate the reward penalty ($\kappa_1(v)/\kappa_2(v)/\kappa_\infty$)  that takes $O(S)$ iterations. As calculation of mean, variance and median, all are linear time compute.  Hence the complexity is
     \[O\Bigm(\log(\frac{1}{\epsilon})\bigm(S^2A + S\bigm)\Bigm) = O\Bigm(\log(\frac{1}{\epsilon})S^2A\Bigm).\]
     
  \end{enumerate}
   
\subsection{Exact Value iteration: Top k response}
In this section, we discuss the time complexity of $\mathtt{s}$-rectangular $L_1/L_2$ robust MDPs as in algorithm \ref{alg:SLp}.  We need to calculate the reward penalty ($\kappa_\infty(v)/\kappa_2(v)$ in \eqref{alg:SLP:eq:kappa}) that  takes $O(S)$ iterations. Then for each state we do: sorting of Q-values in \eqref{alg:SLP:eq:Qsort}, value evaluation in \eqref{alg:SLP:eq:valEval}, update Q-value in \eqref{alg:SLP:eq:Qupdate} that takes $O(A\log(A)), O(A), O(SA)$ iterations respectively. 
    Hence the complexity is
    \[\text{total iteration(reward cost \eqref{alg:SLP:eq:kappa} + S( sorting \eqref{alg:SLP:eq:Qsort} + value evaluation \eqref{alg:SLP:eq:valEval} +Q-value\eqref{alg:SLP:eq:Qupdate})}\]
    \[=\log(\frac{1}{\epsilon})(S + S(A\log(A) + A +SA)\]
     \[O\Bigm(\log(\frac{1}{\epsilon})\bigm(S^2A + SA\log(A)\bigm)\Bigm).\]

For general $p$, we need little caution as $k_p(v)$ can't be calculated exactly but approximately by binary search. And it is the subject of discussion for the next sections.

\subsection{Inexact Value Iteration: sa-rectangular Lp robust MDPs ($\mathcal{U}^{\mathtt{sa}}_p$)}
In this section, we will study the time complexity for robust value iteration for $(\mathtt{sa})$-rectangular $L_p$ robust MDPs for general $p$. Recall, that value iteration takes best penalized action, that is easy to compute. But reward penalization depends on $p$-variance measure $\kappa_p(v)$, that we will estimate by $\hat{\kappa}_p(v)$ through binary search. We have inexact value iterations  as 
\[v_{n+1}(s) := \max_{a\in\mathcal{A}}[\alpha_{sa} - \gamma\beta_{sa}\hat{\kappa}_q(v_n) + R_0(s,a) + \gamma\sum_{s'}P_0(s'|s,a)v_n(s')]\]
where $\hat{\kappa}_q(v_n)$ is a $\epsilon_1$ approximation of $\kappa_q(v_n)$, that is $\lvert\hat{\kappa}_q(v_n) - \kappa_q(v_n)\rvert \leq \epsilon_1$. Then it is easy to see that we have bounded error in robust value iteration, that is 
\[\lVert v_{n+1} -\mathcal{T}^*_{\mathcal{U}^{\mathtt{sa}}_p}v_n\rVert_{\infty} \leq \gamma\beta_{max}\epsilon_1\]
where $\beta_{max} :=\max_{s,a}\beta_{s,a}$
\begin{proposition}\label{rs:appVI} Let $\mathcal{T}^*_{\mathcal{U}}$ be a $\gamma$ contraction map, and $v^*$ be its fixed point. And let $\{v_n,n\geq 0\}$ be approximate value iteration, that is 
\[ \lVert v_{n+1} -\mathcal{T}^*_{\mathcal{U}}v_n\rVert_\infty \leq\epsilon \]
then 
\[\lim_{n\to \infty}\lVert v_{n} -v^*\rVert_\infty \leq \frac{\epsilon}{1-\gamma}\]
moreover, it converges to the $\frac{\epsilon}{1-\gamma}$ radius ball linearly, that is
\[\lVert v_{n} -v^*\rVert_\infty-\frac{\epsilon}{1-\gamma} \leq c\gamma^n\]
where $c = \frac{1}{1-\gamma }\epsilon + \lVert v_0 -v^*\rVert_\infty$.
\end{proposition}
\begin{proof}
\begin{equation}\begin{aligned}
    \lVert v_{n+1} -v^*\rVert_\infty = & \lVert v_{n+1} -\mathcal{T}^*_{\mathcal{U}}v^*\rVert_\infty\\
    =&\lVert v_{n+1}- \mathcal{T}^*_{\mathcal{U}}v_n +\mathcal{T}^*_{\mathcal{U}}v_n -\mathcal{T}^*_{\mathcal{U}}v^*\rVert_\infty\\
    \leq&\lVert v_{n+1}- \mathcal{T}^*_{\mathcal{U}}v_n\rVert_\infty +\lVert \mathcal{T}^*_{\mathcal{U}}v_n -\mathcal{T}^*_{\mathcal{U}}v^*\rVert_\infty\\
    \leq&\lVert v_{n+1}- \mathcal{T}^*_{\mathcal{U}}v_n\rVert_\infty +\gamma\lVert v_n -v^*\rVert_\infty, \qquad \text{(contraction)}\\
    \leq&\epsilon +\gamma\lVert v_n -v^*\rVert_\infty, \qquad \text{(approximate value iteration)}\\
    \implies \lVert v_n -v^*\rVert_\infty = &\sum_{k=0}^{n-1}\gamma^k\epsilon + \gamma^n\lVert v_0 -v^*\rVert_\infty, \qquad \text{(unrolling above recursion)}\\
     = &\frac{1-\gamma^n}{1-\gamma }\epsilon + \gamma^n\lVert v_0 -v^*\rVert_\infty\\
     = &\gamma^n[\frac{1}{1-\gamma }\epsilon + \lVert v_0 -v^*\rVert_\infty] +\frac{\epsilon}{1-\gamma } 
\end{aligned}\end{equation}
Taking limit $n \to \infty$ both sides, we get 
\[\lim_{n\to \infty}\lVert v_{n} -v^*\rVert_\infty \leq \frac{\epsilon}{1-\gamma}.\]
\end{proof}
\begin{lemma}\label{rs:saLpCompl}For $\mathcal{U}^{\mathtt{sa}}_p$, the total iteration cost is $\log(\frac{1}{\epsilon})S^2A+ (\log(\frac{1}{\epsilon}))^2$ to get $\epsilon$ close to the optimal robust value function.
\end{lemma}
\begin{proof}
We calculate $\kappa_q(v)$ with $\epsilon_1 = \frac{(1-\gamma)\epsilon}{3}$ tolerance that takes $O(S\log(\frac{S}{\epsilon_1}))$ using binary search (see section \ref{app:BSkappa}). Now, we do approximate value iteration for $ n =\log(\frac{3\lVert v_0 -v^*\rVert_\infty}{\epsilon})$. Using the above lemma, we have
\begin{equation}\begin{aligned}
    \lVert v_{n} -v^*_{\mathcal{U}^{\mathtt{sa}}_p}\rVert_\infty= &\gamma^n[\frac{1}{1-\gamma }\epsilon_1 + \lVert v_0 -v^*_{\mathcal{U}^{\mathtt{sa}}_p}\rVert_\infty] +\frac{\epsilon_1}{1-\gamma } \\
    \leq &\gamma^n[\frac{\epsilon}{3} + \lVert v_0 -v^*_{\mathcal{U}^{\mathtt{sa}}_p}\rVert_\infty] +\frac{\epsilon}{3} \\
    \leq &\gamma^n\frac{\epsilon}{3} + \frac{\epsilon}{3} +\frac{\epsilon}{3} 
    \leq \epsilon.
\end{aligned}\end{equation}
In summary, we have action cost $O(A)$, reward cost $O(S\log(\frac{S}{\epsilon}))$, sweep cost $O(S)$ and total number of iterations $O(\log(\frac{1}{\epsilon}))$. So the complexity is 
\[\text{(number of iterations)\big(S(actions cost) (sweep cost) + reward cost\big)}\]
\[= \log(\frac{1}{\epsilon})\bigm(S^2A +   S\log(\frac{S}{\epsilon})\bigm) = \log(\frac{1}{\epsilon}) (S^2 A + S\log(\frac{1}{\epsilon}) +S\log(S) )\]
\[ = \log(\frac{1}{\epsilon})S^2A+ S(\log(\frac{1}{\epsilon}))^2 \]
\end{proof}

\subsection{Inexact Value Iteration: s-rectangular Lp robust MDPs }
In this section, we study the time complexity for robust value iteration for \texttt{s}-rectangular $L_p$ robust MDPs for general $p$ ( algorithm \ref{alg:SALp}). Recall, that value iteration takes regularized actions and penalized reward. And reward penalization depends on $q$-variance measure $\kappa_q(v)$, that we will estimate by $\hat{\kappa}_q(v)$ through binary search, then again we will calculate $\mathcal{T}^*_{\mathcal{U}^{\mathtt{sa}}_p}$ by binary search with approximated $\kappa_q(v)$. Here, we have two error sources (\eqref{alg:SLP:eq:kappa}, \eqref{alg:SLP:eq:valEval}) as contrast to $(\mathtt{sa})$-rectangular cases, where there was only one error source from the estimation of $\kappa_q$.

First, we account for the error caused by the first source ($\kappa_q$). Here we do value iteration with approximated $q$-variance $\hat{\kappa}_q$, and exact action regularizer.
We have 
\[v_{n+1}(s) :=\lambda \quad \text{s.t. }\quad \alpha_s +\gamma\beta_{s}\hat{\kappa}_q(v) = (\sum_{Q(s,a)\geq \lambda}(Q(s,a) - \lambda)^{p})^{\frac{1}{p}}\]
where $Q(s,a) =R_0(s,a) + \gamma\sum_{s'}P_0(s'|s,a)v_n(s'),$ and $ \lvert\hat{\kappa}_q(v_n) - \kappa_q(v_n)\rvert \leq \epsilon_1$. Then from the next result (proposition \ref{rs:sLpkappaErr}), we get
\[\lVert v_{n+1} -\mathcal{T}^*_{\mathcal{U}^{\mathtt{sa}}_p}v_n\rVert_\infty \leq \gamma\beta_{max}\epsilon_1\]
where $\beta_{max} :=\max_{s,a}\beta_{s,a}$

\begin{proposition}\label{rs:sLpkappaErr} Let $\hat{\kappa}$ be an an $\epsilon$-approximation of $\kappa$, that is  $ \lvert \hat{\kappa} - \kappa\rvert \leq \epsilon$, and let $b\in\mathbb{R}^A$ be sorted component wise, that is, $b_1\geq, \cdots,\geq b_A$. Let $\lambda$ be the solution to the following equation with exact parameter $\kappa$,
\[\alpha +\gamma\beta\kappa = (\sum_{b_i\geq \lambda}|b_i - \lambda|^{p})^{\frac{1}{p}}\] and let $\hat{\lambda}$ be the solution of the following equation with approximated parameter $\hat{\kappa}$,
\[\alpha +\gamma\beta\hat{\kappa} = (\sum_{b_i\geq \hat{\lambda}}|b_i - \hat{\lambda}|^{p})^{\frac{1}{p}},\] 
then $\hat{\lambda}$ is an $O(\epsilon)$-approximation of $\lambda$, that is 
\[\lvert \lambda - \hat{\lambda}\rvert \leq \gamma\beta\epsilon.\]
\end{proposition}
\begin{proof}
Let the function $f:[b_A,b_1]\to\mathbb{R}$ be defined as  
\[f(x) := (\sum_{b_i\geq x}|b_i - x|^{p})^{\frac{1}{p}}.\]
We will show that derivative of $f$ is bounded, implying its inverse is bounded and hence Lipschitz, that will prove the claim. Let proceed 
\begin{equation}\begin{aligned}
    \frac{df(x)}{dx} &= -(\sum_{b_i\geq x}|b_i - x|^{p})^{\frac{1}{p} -1}\sum_{b_i\geq x}|b_i - x|^{p-1}\\
     &= - \frac{\sum_{b_i\geq x}|b_i - x|^{p-1}}{(\sum_{b_i\geq x}|b_i - x|^{p})^{\frac{p-1}{p}}}\\
     &= - \Bigm[\frac{(\sum_{b_i\geq x}|b_i - x|^{p-1})^{\frac{1}{p-1}}}{(\sum_{b_i\geq x}|b_i - x|^{p})^{\frac{1}{p}}}\Bigm]^{p-1}\\
     &\leq -1.
\end{aligned}
\end{equation}
The inequality follows from the following  relation between $L_p$ norm, 
\[\lVert x\rVert_a \geq \lVert x\rVert_b ,\qquad \forall 0\leq  a\leq b.   \]
It is easy to see that the function $f$ is strictly monotone in the range $b_A ,b_1]$, so its inverse is well defined in the same range. Then derivative of the inverse of the function $f$ is bounded as
\[ 0 \geq \frac{d}{dx}f^-(x)\geq -1.\]
Now, observe that $\lambda = f^-(\alpha +\gamma\beta\kappa)$ and $\hat{\lambda} = f^-(\alpha +\gamma\beta\hat{\kappa})$, then by Lipschitzcity, we have 
\[|\lambda - \hat{\lambda}| = |f^-(\alpha +\gamma\beta\kappa)- f^-(\alpha +\gamma\beta\hat{\kappa})| \leq \gamma\beta|-\kappa -\hat{\kappa})| \leq \gamma\beta\epsilon. \]
\end{proof}

\begin{lemma}For $\mathcal{U}^{\mathtt{s}}_p$, the total iteration cost is $O\Bigm(\log(\frac{1}{\epsilon})\bigm( S^2A+ SA\log(\frac{A}{\epsilon}) \bigm)\Bigm)$ to get $\epsilon$ close to the optimal robust value function.
\end{lemma}
\begin{proof}
We calculate $\kappa_q(v)$ in \eqref{alg:SLP:eq:kappa} with $\epsilon_1 = \frac{(1-\gamma)\epsilon}{6}$ tolerance that takes $O(S\log(\frac{S}{\epsilon_1}))$ iterations using binary search (see section \ref{app:BSkappa}). Then for every state, we sort the $Q$ values (as in \eqref{alg:SLP:eq:Qsort}) that costs $O(A\log(A))$ iterations. In each state, to update value, we do again binary search with approximate $\kappa_q(v)$ upto $\epsilon_2:=\frac{(1-\gamma)\epsilon}{6}$ tolerance, that takes $O(\log(\frac{1}{\epsilon_2}))$ search iterations and each iteration cost $O(A)$, altogether it costs $O(A\log(\frac{1}{\epsilon_2}))$ iterations. Sorting of actions and binary search adds upto $O(A\log(\frac{A}{\epsilon}))$ iterations (action cost).
So we have (doubly) approximated value iteration  as following, 
\begin{equation}
    \lvert v_{n+1}(s)  -\hat{\lambda}\rvert \leq \epsilon_1
\end{equation}
where
\[(\alpha_s +\gamma\beta_{s}\hat{\kappa}_q(v_n))^p = \sum_{Q_n(s,a)\geq \hat{\lambda}}(Q_n(s,a) - \hat{\lambda})^{p}\]
and  \[Q_n(s,a) =R_0(s,a) + \gamma\sum_{s'}P_0(s'|s,a)v_n(s'), \qquad \lvert \hat{\kappa}_q(v_n) -\kappa_q(v_n)\rvert \leq \epsilon_1.\] 
And we do this approximate value iteration for $ n =\log(\frac{3\lVert v_0 -v^*\rVert_\infty}{\epsilon})$. Now, we do error analysis. By accumulating error, we have
\begin{equation}\begin{aligned}
    \lvert v_{n+1}(s) -(\mathcal{T}^*_{\mathcal{U}^{\mathtt{s}}_p}v_n)(s)\rvert \leq& \lvert v_{n+1}(s) -\hat{\lambda}\rvert + \lvert \hat{\lambda}-(\mathcal{T}^*_{\mathcal{U}^{\mathtt{s}}_p}v_n)(s)\rvert\\
    \leq& \epsilon_1 +\lvert\hat{\lambda}-(\mathcal{T}^*_{\mathcal{U}^{\mathtt{s}}_p}v_n)(s)\rvert , \qquad \text{(by definition)}\\
    \leq &\epsilon_1 + \gamma\beta_{\max}\epsilon_1, \qquad \text{(from proposition \ref{rs:sLpkappaErr})}\\
    \leq& 2\epsilon_1.
\end{aligned}\end{equation}

where $\beta_{max}:=\max_{s}\beta_s, \gamma \leq 1$.

Now, we do approximate value iteration, and from proposition \ref{rs:appVI}, we get 
\begin{equation}\begin{aligned}
    \lVert v_{n} -v^*_{\mathcal{U}^s_p}\rVert
    \leq& \frac{2\epsilon_1}{1-\gamma} + \gamma^n[\frac{1}{1-\gamma }2\epsilon_1 + \lVert v_0 -v^*_{\mathcal{U}^s_p} \rVert_\infty]
\end{aligned}\end{equation}
Now, putting the value of $n$, we have
\begin{equation}\begin{aligned}
    \lVert v_{n} -v^*_{\mathcal{U}^s_p}\rVert_\infty= &\gamma^n[\frac{2\epsilon_1}{1-\gamma } + \lVert v_0 -v^*_{\mathcal{U}^s_p}\rVert_\infty] +\frac{2\epsilon_1}{1-\gamma } \\
    \leq &\gamma^n[\frac{\epsilon}{3} + \lVert v_0 -v^*_{\mathcal{U}^s_p}\rVert_\infty] +\frac{\epsilon}{3} \\
    \leq &\gamma^n\frac{\epsilon}{3} + \frac{\epsilon}{3} +\frac{\epsilon}{3} 
    \leq \epsilon.
\end{aligned}\end{equation}
To summarize, we do $O(\log(\frac{1}{\epsilon}))$ full value iterations. Cost of evaluating reward penalty is $O(S\log(\frac{S}{\epsilon}))$. For each state: evaluation of $Q$-value from value function requires $O(SA)$ iterations, sorting the actions according $Q$-values requires $O(A\log(A))$ iterations, and binary search for evaluation of value requires $O(A\log(1/\epsilon)$. So the complexity is 
\[O(\text{(total iterations)(reward cost + S(Q-value + sorting + binary search for value )}))\]
\[= O\Bigm(\log(\frac{1}{\epsilon})\bigm(S\log(\frac{S}{\epsilon}) + S(SA +A\log(A)+ A\log(\frac{1}{\epsilon})) \bigm)\Bigm)\]
\[= O\Bigm(\log(\frac{1}{\epsilon})\bigm(S\log(\frac{1}{\epsilon}) + 
S\log(S) + S^2A+SA\log(A) + SA\log(\frac{1}{\epsilon}) \bigm)\Bigm)\]
\[= O\Bigm(\log(\frac{1}{\epsilon})\bigm( S^2A+SA\log(A) + SA\log(\frac{1}{\epsilon}) \bigm)\Bigm)\]

\[= O\Bigm(\log(\frac{1}{\epsilon})\bigm( S^2A+ SA\log(\frac{A}{\epsilon}) \bigm)\Bigm)\]
\end{proof}

\end{document}